\documentclass[11pt,english]{article}
\usepackage[T1]{fontenc}
\usepackage[latin9]{inputenc}
\usepackage{amsmath}
\usepackage{amsthm}
\usepackage{amssymb}
\usepackage{bm}
\usepackage{graphicx}
\usepackage{booktabs} 
\usepackage{caption}
\usepackage{subcaption}
\usepackage{comment}
\usepackage{microtype} 
\usepackage{verbatim}
\usepackage[title]{appendix}
\usepackage{algpseudocode}
\usepackage{algorithm}
\usepackage{enumerate}
\usepackage[dvipsnames]{xcolor}
\definecolor{Gray}{gray}{0.94}
\usepackage[align=center,shadow=true,shadowsize=6pt,nobreak=true,framemethod=tikz,skipabove=9.5pt,skipbelow=9pt,innertopmargin=5pt,innerbottommargin=5pt,innerleftmargin=5pt,innerrightmargin=5pt,leftmargin=1.5pt,rightmargin=1.5pt]{mdframed}
\usetikzlibrary{shadows}

\makeatletter
\numberwithin{equation}{section}
\numberwithin{figure}{section}
\theoremstyle{plain}
\newtheorem{thm}{\protect\theoremname}
\theoremstyle{definition}

\theoremstyle{plain}
\newtheorem{lemma}{\protect\lemmaname}
\newtheorem{corollary}{\protect\corollaryname}

\theoremstyle{remark}
\newtheorem*{rem*}{\protect\remarkname}
\newtheorem{assumption}{\protect\assumptionname}
\theoremstyle{plain}
\newtheorem{manualtheoreminner}{Theorem}
\newenvironment{manualtheorem}[1]{%
  \renewcommand\themanualtheoreminner{#1}%
  \manualtheoreminner
}{\endmanualtheoreminner}

\@ifundefined{date}{}{\date{}}
\usepackage{color}
\usepackage{colortbl}
\usepackage{hyperref}
\hypersetup{
    colorlinks,
    allcolors=[rgb]{0.3,0.1,0.9}
}

\usepackage[lined,boxed,ruled, linesnumbered, algo2e]{algorithm2e}

\usepackage[margin=1in]
           {geometry}

\usepackage{thmtools}
\usepackage{thm-restate}
\usepackage{amsfonts}
\usepackage{tikz}
\usepackage{enumitem}
\usepackage{wrapfig}
\allowdisplaybreaks

\makeatother

\usepackage{babel}
\providecommand{\definitionname}{Definition}
\providecommand{\lemmaname}{Lemma}
\providecommand{\remarkname}{Remark}
\providecommand{\theoremname}{Theorem}
\providecommand{\corollaryname}{Corollary}
\providecommand{\remarkname}{Remark}
\providecommand{\assumptionname}{Assumption}

\begin{document}
\global\long\def\norm#1{\left\Vert #1\right\Vert }%
\global\long\def\R{\mathbb{R}}%
\global\long\def\eps{\epsilon}%
 
\global\long\def\Rn{\mathbb{R}^{n}}%
\global\long\def\tr{\mathrm{Tr}}%
\global\long\def\diag{\mathrm{diag}}%
\global\long\def\Diag{\mathrm{Diag}}%
\global\long\def\C{\mathbb{C}}%
\global\long\def\conv{\mathrm{conv}}%
 
\global\long\def\E{\mathbb{E}}%
\global\long\def\vol{\mathrm{vol}}%
\global\long\def\argmax{\mathrm{argmax}}%
\global\long\def\argmin{\mathrm{argmin}}%
\global\long\def\sign{\mathrm{sign}}%
\global\long\def\bd{\mathrm{bd}}%

\global\long\def\ham{\mathrm{Ham}}%
\global\long\def\e#1{ \exp\left(#1\right)}%
\global\long\def\Var{\mathrm{Var}}%
\global\long\def\dint{{\displaystyle \int}}%
\global\long\def\step{\delta}%
\global\long\def\Ric{\mathrm{Ric}}%
\global\long\def\P{\mathbb{P}}%
\global\long\def\len{\text{\text{len}}}%
\global\long\def\lspan{\mathrm{span}}%
\newcommand{\squeezeup}{\vspace{-2.5mm}}
\newcommand{\eg}{\emph{e.g.}}
\newcommand{\ie}{\emph{i.e.}}

\bibliographystyle{alpha}

\title{Provable Lifelong Learning of Representations}

\author{Xinyuan Cao \\ Georgia Tech\\\texttt{xcao78@gatech.edu} \and Weiyang Liu \\ University of Cambridge \& MPI-IS\\\texttt{wl396@cam.ac.uk}  \and Santosh S. Vempala \\ Georgia Tech\\\texttt{vempala@gatech.edu}}
\maketitle
\begin{abstract}
In lifelong learning, tasks (or classes) to be learned arrive sequentially over time in arbitrary order. During training, knowledge from previous tasks can be captured and transferred to subsequent ones to improve sample efficiency. We consider the setting where all target tasks can be represented in the span of a small number of unknown linear or nonlinear features of the input data. We propose a lifelong learning algorithm that maintains and refines the internal feature representation. We prove that for any desired accuracy on all tasks, the dimension of the representation remains close to that of the underlying representation. The resulting sample complexity improves significantly on existing bounds. In the setting of linear features, our algorithm is provably efficient and the sample complexity for input dimension $d$, $m$ tasks with $k$ features up to error $\epsilon$ is $\tilde{O}(dk^{1.5}/\epsilon+km/\epsilon)$. We also prove a matching lower bound for any lifelong learning algorithm that uses a single task learner as a black box. We complement our analysis with an empirical study, including a heuristic lifelong learning algorithm for deep neural networks. Our method performs favorably on challenging realistic image datasets compared to state-of-the-art continual learning methods.

\end{abstract}

\section{Introduction}

Recent years have witnessed significant advances in both theory and practice of supervised learning. While a variety of techniques are available for learning individual target functions, much less is known about {\em continual} or {\em lifelong} learning, where the learner is adding new target functions to their repertoire. Inspired by how humans learn and transfer knowledge during their lifespan, lifelong learning has many applications in computer vision~\cite{parisi2019continual} and robotics~\cite{thrun1995lifelong}.

A central idea for lifelong learning is to learn an {\em efficient representation} that facilitates the collection of target functions to be learned. For example, if deep feed-forward networks are being used for classification, the goal might be to learn a hidden layer whose outputs are relevant and useful features for the family of tasks. Building a classifier on top of them is relatively easy or less expensive than building one from the original input features. This representation itself is incrementally refined as more target functions are learned. 

We consider a very general setting of task/class incremental learning, where new samples from different tasks/classes are presented sequentially over time. The goal of the learner is to maintain hypothesis functions that work for all tasks/classes encountered so far. We assume that all targets are simple functions of a bounded number of unknown {\em linear or nonlinear} features.

Prior work~\cite{balcan2015efficient} considered the task-incremental setting where the target functions are linear classifiers of the input that all lie in a common low-dimensional subspace. Under this assumption, a simple algorithm can be shown to learn a good representation of size comparable to the optimal one (\ie, a basis of the common low-dimensional subspace). The algorithm proceeds as follows: maintain a small number of linear features; learn the next function as a linear function of the features; if the error is too high, learn the new function directly on the input, and add it as a new feature. Under mild assumptions on the input distribution (log-concavity), with a suitable choice of error parameters, this algorithm is guaranteed to learn a small set of features that work well for all the target functions. More recent works \cite{du2020few,tripuraneni2020provable,chua2021fine} focus on the sample complexity of multi-task learning under strong distributional assumptions on both the data and the tasks. 

Our paper is motivated by the following questions:

\begin{mdframed}
\begin{itemize}[leftmargin=7mm]
\setlength\itemsep{0.04em}
\item Can the theoretical guarantees for linear features be extended to a representation with nonlinear features?
\item Does the refinement of the internal representation have provable benefits?
\item What is the best possible sample complexity of lifelong learning?
\end{itemize}
\end{mdframed}

Our work addresses these questions for both task-incremental learning (classification or regression) and class-incremental learning (where we do not have access to the task ID). Our analysis applies to a broad class of lifelong learning algorithms that dynamically change the network architecture. First, we analyze the setting where the underlying common features are nonlinear, which is considerably more general than previously considered. We prove that this natural lifelong learning algorithm is guaranteed to learn low-error targets creating only a small number of nonlinear features. Secondly, we propose a new algorithm, with a refinement step, and show that it improves the sample complexity using a new perspective on feature subspaces. The resulting sample complexity improves significantly on known bounds for the setting of linear features, and perhaps surprisingly, we show that it is the best possible in the setting of linear features, assuming that the lifelong learner has black-box access to a single task learner to any desired level of accuracy. 
We do this by constructing a hard distribution over tasks. 

We conduct experiments on class-incremental learning using benchmark data sets and find that our proposed algorithm outperforms state-of-the-art continual learning algorithms.

\subsection{Problem Settings}

We consider $m$ tasks (or $m$-class classification) where the tasks (classes) arrive sequentially over time. Let $X=\mathbb{R}^d$ be the input space and $Y$ be the label space. We study a discriminative model, where the target function of each task can be learned using a linear combination of at most $k$ linear/nonlinear features. The goal is to learn a hypothesis function with small generalization errors on all tasks.

Formally, the problem is associated with a distribution $P$ over $X\times Y$, $D$ is the marginal of $P$ over $X$. The label for an input data point $\bm{x}\in\mathbb{R}^d$ is given by 
\[
\ell(\bm{x})=\phi\left(\langle {\bm{c}^*},\bm{\sigma}^*\left(\bm{x}\right)\rangle\right)
\]
where $\bm{\sigma}^*(\bm{x})=(\sigma_1^*(\bm{x}), \ldots, \sigma_k^*(\bm{x}))^\top\in\mathbb{R}^k$ is a vector of unknown features, $\bm{c}^*\in\mathbb{R}^k$, $\phi(\cdot): \mathbb{R} \rightarrow Y$ is the map to the label space. $(k \ll \min (m,d))$. Equivalently, we can view it as a two-layer network with $k$ neurons in the hidden layer (Figure \ref{fig:sketch_lll}).

\begin{figure}[tp]
    \centering
    \begin{subfigure}[b]{0.49\textwidth}
         \centering
         \includegraphics[height=2.2in]{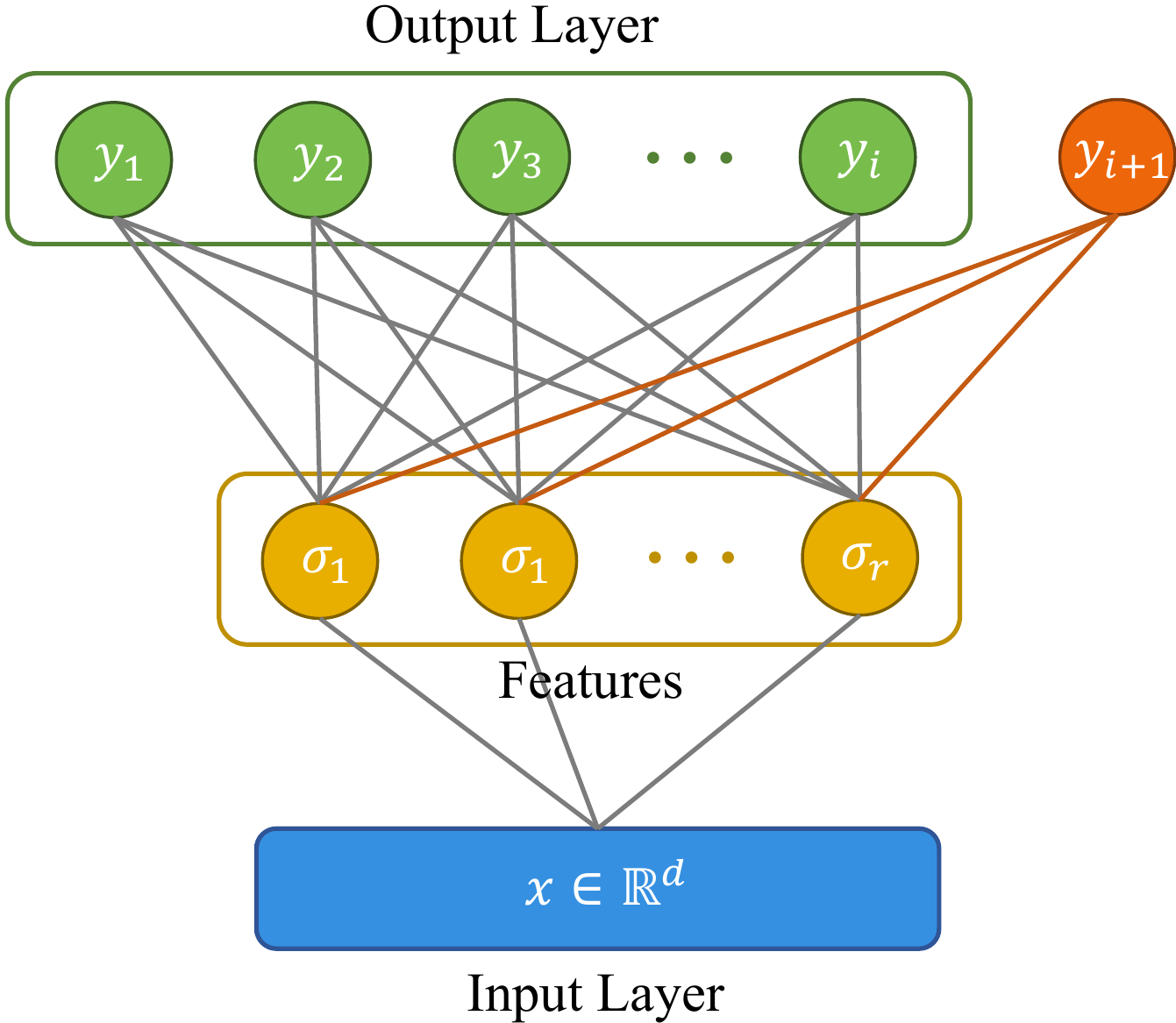}
         \caption{\textbf{Case1}: small error with current features}
     \end{subfigure}
     \begin{subfigure}[b]{0.49\textwidth}
         \centering
         \includegraphics[height=2.2in]{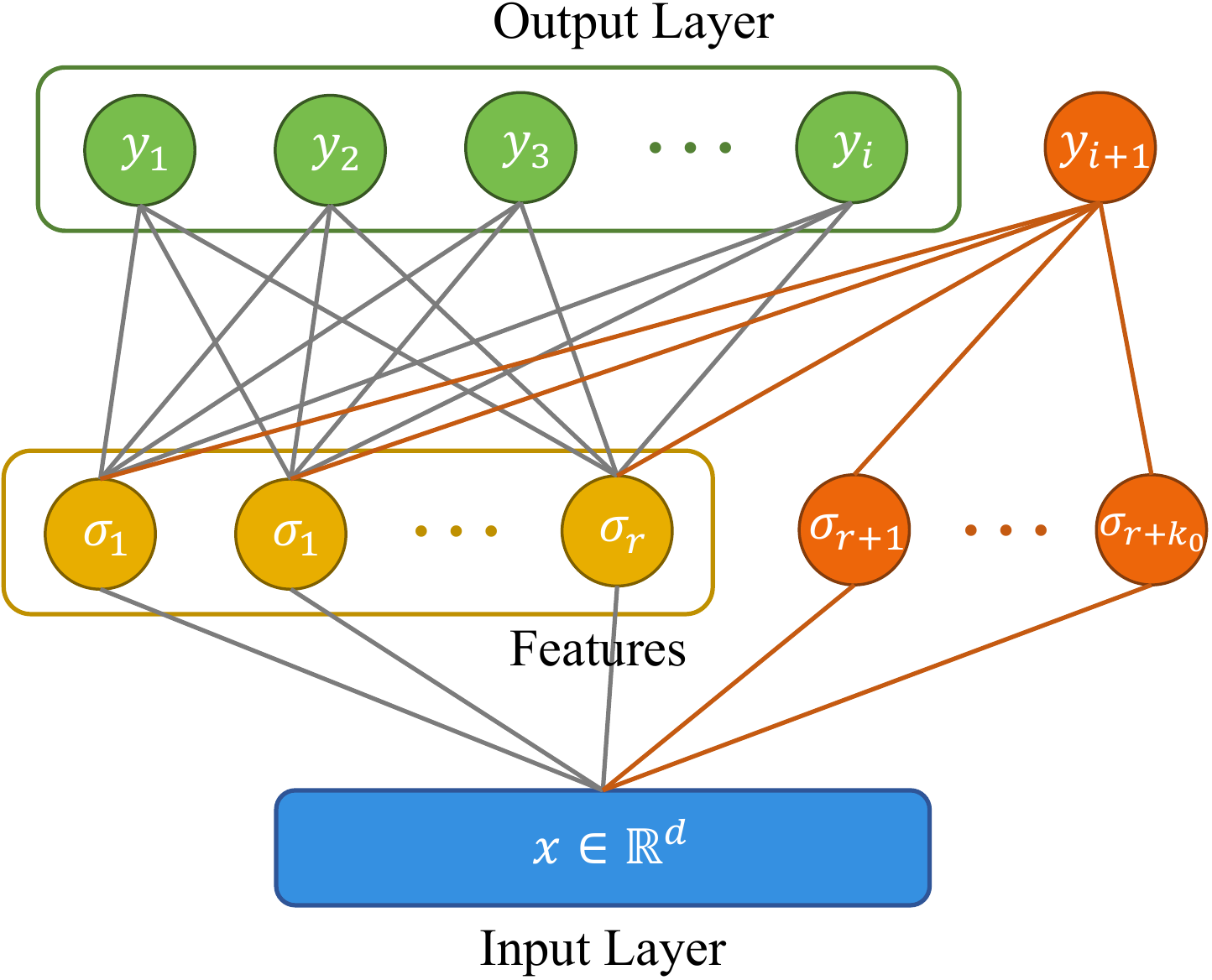}
         \caption{\textbf{Case2}: large error with current features}
     \end{subfigure}
    \caption{An illustration of LLL. Given a new task $y_{i+1}$, the algorithm tries to learn the task with existing features $\sigma_1,\cdots,\sigma_r$. If the error is small (case1, Sub-figure (a)), then it moves to the next task. Otherwise (case2, Sub-figure (b)), it learns a new set of features  $\sigma_{r+1},\cdots,\sigma_{r+k_0}$ and a linear combination of all features; for linear features, it learns a single new feature $\sigma_{r+1}$.}
    \label{fig:sketch_lll}
\end{figure}

Our goal is to learn a good hypothesis function $\hat{\ell}(\cdot)$ parameterized by $(\bm{c}^*, \bm{\sigma}^*)$ with a small generalization error $err = \mathbb{P}_{(\bm{x},y)\sim P}L(l(\bm{x}), \hat{l}(\bm{x}))$, where $L(\cdot,\cdot)$ is some loss function for the specific task.

We use a similar model with multi-task learning, where all tasks share the same low-dimensional feature subspace. However, it is different from multi-task learning, which has $T_1$ source tasks to learn all-at-once and use the features learned to solve the target tasks. The assumption is made there that the features of the target task are covered by all features that have been learned. Instead, lifelong learning algorithms learn all tasks sequentially, with no prior knowledge of the incoming tasks during training.

Here we focus on the task-incremental learning of binary classification tasks. Extensions to task-incremental learning of linear regression and multi-class classification tasks are given in Appendix~\ref{section:class_increment}.

Let $\bm{X}=\mathbb{R}^d$ be the input space, $\bm{Y}=\{\pm 1\}$ be the label space. For any task $i\in [m]$, any sample $(\bm{x}, y)$ drawn from $P$ satisfies $y=l_i(\bm{x})=\sign(\langle \bm{c}^*_i, \bm{\sigma}^*(\bm{x})\rangle)$, where the features are $\bm{\sigma}^*(\bm{x})= \bm{W}^*\bm{x}$ in the linear case and $\bm{\sigma}^*(\bm{x})=f(\bm{W}^*\bm{x})$ in the nonlinear case. Here $f(\cdot)$ is a nonlinear activation function, \eg ReLU. $\bm{W}^*\in\mathbb{R}^{k\times d}, \bm{c}^*_i\in\mathbb{R}^k$ etc. Specifically, in the linear case, for each task $i$, we equivalently have $y=\sign(\langle \bm{a}^*_i, \bm{x}\rangle)$, where $\bm{a}^*_i={\bm{W}^*}^\top \bm{c}^*_i \in\mathbb{R}^d$. WLOG we assume that each $\bm{a}_i^*$ is a unit vector, \ie, $\|\bm{a}^*_i\|_2=1, \forall i\in[m]$. The generalization error is defined as $err = \mathbb{P}_{(\bm{x},y)\sim P}(l_i(\bm{x})\neq \hat{l}_i(\bm{x}))$.

\subsection{Main Results}

In all the results and analysis, we only have Assumption \ref{equation:assumption}, which as Lemma 1 from  \cite{balcan2015efficient}) asserts, is satisfied by all log-concave distributions after an affine transformation. This class includes many common distributions, such as Gaussian, Uniform and Gamma distributions \cite{lovasz2007geometry}.

\begin{assumption}\label{equation:assumption}
(Data Distribution Assumption) Let $\theta(\cdot,\cdot)$ denote the angle between two vectors. We assume that there exist universal constants $c_2>c_1 > 0$ s.t., for any unit vectors $\bm{u},\bm{v} \in \R^d$, 
\begin{equation*}
    c_1 \theta(\bm{u},\bm{v})\leq \mathbb{P}_{\bm{x} \sim D}(\sign(\bm{u}\cdot \bm{x})\neq \sign(\bm{v}\cdot \bm{x})) \leq c_2 \theta (\bm{u},\bm{v})
\end{equation*}
\end{assumption}

Our theoretical upper bounds are summarized below. The detailed statements for these results appear as Theorem~\ref{thm:vanilla_algo}, Theorem~\ref{thm:tune_linear} and  Theorem \ref{thm:sdp_approx} in Section~\ref{section:theoretical_guarantees}. These results are based on the algorithms described in Section~\ref{section:algorithms}, called Basic Lifelong Learning (LLL) and Lifelong Learning with Representation Refinement (LLL-RR).

\begin{mdframed}
\begin{thm}[Summary of Upper Bounds]\label{thm:informal}
\vspace{-2.5mm}
Consider the lifelong learning setting of input dimension $d$, $m$ tasks with $k$ common features.
The basic lifelong learning algorithm achieves a target error of $\epsilon$ on all tasks with sample complexity $\tilde{O}(dk^{1.5}/\epsilon+km/\epsilon)$ for linear features and a factor of $k$ higher for nonlinear features. With representation refinement using at most $2k$ features, the sample complexity is $\tilde{O}(dk^{1.5}/\epsilon+km/\epsilon)$. In the linear setting refinement runs in polynomial-time.
\end{thm}
\end{mdframed}

This raises the question of whether there exist algorithms with better sample complexity. We show that the answer is NO in a general sense. We assume that we have black-box access to a single-task learner that works as follows: it takes as input labeled examples and a target accuracy $\eps$, and outputs some feasible solution with error at most $\eps$. Then we show that any lifelong learning algorithm that achieves $\eps$ error for all tasks needs 
$\Omega(dk^{1.5}/\eps+km/\eps)$ samples. 
\begin{mdframed}
\begin{thm}[Lower Bound]\label{thm:adv_lower_bound}
\vspace{-2.5mm}
Suppose that a lifelong learner has black-box access to a single task learner that takes an error parameter $\eps$ as input and is allowed to return any vector that is within distance $\eps$ of the true target unit vector, using $\Theta(d/\eps)$ samples in $\R^d$. 
Then, there exists a distribution of $m$ tasks, $m = 2^{\Theta(k)}$ such that for any lifelong learning algorithm, WHP, the total number of samples required to learn all $m$ tasks up to error $\eps$ is $\Omega(dk^{1.5}/\eps+km/\eps)$.
\end{thm}
\end{mdframed}

Our contributions can be summarized as follows:

\paragraph{Sample complexity.} We bound the sample complexity of lifelong learning for both the linear and nonlinear cases. In the linear case, our bound for the lifelong learning is $\tilde{O}(dk^{1.5}/\eps+km/\eps)$. This improves the dependence on both $k$ and $\eps$ compared to past work~\cite{balcan2015efficient}, which proved a bound of $\tilde{O}(dk^2/\epsilon^2+km/\epsilon)$. It also improves existing theoretical results for multi-task learning~\cite{du2020few,tripuraneni2020provable}, where the best sample complexity is $\tilde{O}(dk^2/\epsilon+km/\epsilon)$. Moreover, this bound is the best possible up to logarithmic factor for any lifelong learning algorithm.

\paragraph{Representation refinement.}
We propose and analyze the step of sample-free representation refinement in the lifelong learning setting. Specifically, this step aims to reduce the dimension of the feature subspace while keeping the subspace close to the true one. In the linear setting, we provide an algorithmically efficient approach via an SDP relaxation. 
To the best of our knowledge, we provide the first provable bound for representation refinement.

\paragraph{Proof techniques.}
Our analysis is based on geometric insights. The test error translates to the distance between the target and learned vectors. To show that our learned feature subspace is close to the true one, we consider the set of candidate $k$-dimensional subspaces. We would like to show that the measure of this set decreases rapidly during learning. Instead, we identify the set of well-approximated vectors by our current learned subspace and show that the set grows at a geometric rate until it includes all vectors in the true subspace.

\paragraph{Empirical results.}
We evaluate our lifelong learning algorithms on standard benchmarks and compare them with state-of-the-art methods, demonstrating their practice efficiency. We also perform simulations for the setting of linear features, and exhibit results that match our theoretical bounds.

\subsection{Related work}\label{section:related_work}

Lifelong learning \cite{thrun1995lifelong} aims to solve different tasks arriving in a stream, where knowledge from current and previous tasks is re-used in subsequent tasks to improve efficiency and sample complexity. Early works found that lifelong learning can encounter Catastrophic Forgetting (CF) \cite{mccloskey1989catastrophic}, especially when using back-propagation \cite{ratcliff1990connectionist}. That is, the performance on old tasks can drop dramatically after learning a new task. There are three main approaches to addressing this problem: {\em adding a regularization term} \cite{li2017learning, kirkpatrick2017overcoming}, {\em freezing the network from previous tasks and adding branches to new tasks} \cite{xu2018reinforced,rusu2016progressive,liu2019neural,liu2021orthogonal,yoon2017lifelong} and {\em replaying previous tasks' exemplars} \cite{rebuffi2017icarl}. Our work is closest to the second approach in that we dynamically change the architecture to overcome CF. Although we do not know the number of tasks in advance, we prove that our algorithm has a small model size and efficient sample complexity.

Despite a vast literature on lifelong learning methods, theoretical investigations are relatively few. \cite{yin2020optimization} studies the optimization and generalization properties of the regularization-based method by analyzing the loss landscape. \cite{bennani2020generalisation,doan2021theoretical} analyze the generalization of the OGD algorithm \cite{farajtabar2020orthogonal} through NTK \cite{jacot2018neural}. \cite{balcan2015efficient} gives an upper bound on the architecture size when we grow the network when training binary classifiers. We improve their bounds by getting nearly tight sample complexity in the linear case and generalize the approach to the nonlinear regime.

Two topics closely related to lifelong learning are meta-learning and transfer learning. There is a line of work where all tasks approximately~\cite{finn2017model,khodak2019adaptive,balcan2019provable,denevi2019online} or conditionally~\cite{wang2020structured,denevi2020advantage,denevi2021conditional} share a common representation. However, our work focuses on the setting where all tasks share one common low-dimensional representation. \cite{chua2021fine} shows the benefits of task-specific fine-tuning, which is fundamentally different from our refinement step. Our refinement step aims to reduce the representation dimension with slight information loss and help to improve the sample complexity of subsequent tasks. This procedure needs no additional data.

There are other works with similar settings to ours where all tasks share one representation. \cite{baxter1997bayesian} bounds the sample complexity to achieve low average error from a Bayesian/information-theoretic point of view. We compare our results with recent work \cite{du2020few,tripuraneni2020provable, balcan2015efficient} in Table~\ref{table:related_work_linear} in the linear case. Previous works on multi-task learning \cite{du2020few,tripuraneni2020provable} need $\tilde{O}(dk^2/\epsilon)$ or more samples from previous tasks to learn the hidden features, while our algorithm needs $\tilde{O}(dk^{1.5}/\epsilon)$ samples. After that, each new task can be learned up to $\epsilon$ error with $\tilde{O}(k/\eps)$ samples. These results illustrate the efficiency of lifelong learning compared to all-at-once training.
Our analysis generalizes to the setting of nonlinear features, \eg, if labels are generated by a two-layer neural network, with $k$ hidden units. We prove our lifelong learning algorithm is sample efficient with only Assumption~\ref{equation:assumption}, which is minimal compared to related work.

\begin{table}[htb]
  \centering
  \begin{tabular}{p{0.3\textwidth}p{0.27\textwidth}p{0.13\textwidth}p{0.14\textwidth}}
    \toprule
    Method     & Assumptions (linear)     & Feature Dim & Total Samples \\
    \midrule
    \begin{minipage}[t]{0.3\textwidth}
    Few-shot learning \cite{du2020few}$^\dagger$ 
    \end{minipage}
    & 
    \begin{minipage}[t]{0.27\textwidth}
        \begin{itemize}[leftmargin=*]
            \setlength\itemsep{-0.2em}
            \item Sub-gaussian input
        \end{itemize}
    \end{minipage}& $k$ & $\tilde{O}(\frac{dk^2}{\epsilon}+\frac{km}{\epsilon})  $\\
    \begin{minipage}[t]{0.3\textwidth}
    Meta learning \cite{tripuraneni2020provable}$^\dagger$ 
    \end{minipage}
    & 
    \begin{minipage}[t]{0.27\textwidth}
        \begin{itemize}[leftmargin=*]
            \setlength\itemsep{-0.2em}
            \item Sub-gaussian input
        \end{itemize}
    \end{minipage}& $k$ & $\tilde{O}(\frac{dk^2}{\epsilon}+\frac{km}{\epsilon})  $\\
    LLL
    \cite{balcan2015efficient}$^\star$ & 
    \begin{minipage}[t]{0.3\textwidth}
        \begin{itemize}[leftmargin=*]
            \setlength\itemsep{-0.2em}
            \item Log-concave input
        \end{itemize}
    \end{minipage}& $k$ & $\tilde{O}(\frac{dk^2}{\epsilon^2}+\frac{km}{\epsilon})  $\vspace{1.5mm}\\
    \rowcolor{Gray}
    LLL (our paper)$^\star$ & 
    \begin{minipage}[t]{0.27\textwidth}
        \begin{itemize}[leftmargin=*]
            \setlength\itemsep{-0.2em}
            \item Well-spread input\\(Assumption~\ref{equation:assumption})
        \end{itemize}
    \end{minipage}& $2k\log(\frac{\log k}{\epsilon})$ & $\tilde{O}(\frac{dk^{1.5}}{\epsilon}+\frac{km}{\epsilon})  $\\
    \rowcolor{Gray}
    LLL-RR (our paper)$^\star$ & 
    \begin{minipage}[t]{0.27\textwidth}
        \begin{itemize}[leftmargin=*]
            \setlength\itemsep{-0.2em}
            \item Well-spread input\\(Assumption~\ref{equation:assumption})
        \end{itemize}
    \end{minipage}& $2k$ & $\tilde{O}(\frac{dk^{1.5}}{\epsilon}+\frac{km}{\epsilon})  $\\
    \bottomrule
  \end{tabular}
  \caption{Comparison of different transfer learning algorithms in the linear setting. $\dagger$: the method trains all source tasks all at once and then uses the representation to train the target tasks. $\star$: all source tasks and target tasks are learned sequentially.}
  \label{table:related_work_linear}
\end{table}

\paragraph{Notation.}
We use bold upper-case letters to refer to matrices (\eg $\bm{X}$) and bold lower-case letters to refer to vectors (\eg $\bm{x}$). We use $[m]=\{1,2,\cdots,m\}$. We use $\tilde{O}$ to hide polylogarithmic factors. $O,\Omega,\Theta$ are standard notations for order of growth.
For any two vectors $\bm{x},\bm{y}$, let $\theta(\bm{x},\bm{y})$ be the angle between them. The angle between a vector $\bm{x}$ and a subspace $\bm{U}$ is defined as $\theta(\bm{x}, \bm{U})=\min\limits_{\bm{u}\in \bm{U}}\theta(\bm{x},\bm{u})$. For two subspaces $\bm{U}, \bm{V}$, define $\theta(\bm{U}, \bm{V})=\max\limits_{\bm{u}\in \bm{U}}\theta(\bm{u}, \bm{V})$. Thus $\theta(\bm{U},\bm{V}) \le \alpha$ iff for all $\bm{u}\in \bm{U}, \exists \bm{v}\in \bm{V}$ s.t. $\theta(\bm{u}, \bm{v}) \leq \alpha$.
We define the distance from a vector $\bm{u}$ to a subspace $\bm{F}$ as the orthogonal distance: $d(\bm{u}, \bm{F})=\min\limits_{\bm{v}\in \bm{F}}\|\bm{u}-\bm{v}\|_2$.
For a distribution $D$ and two vectors $\bm{u},\bm{v}$, we define  $d_D(\bm{u},\bm{v})=\mathbb{P}_{\bm{x} \sim D}(\sign(\bm{u}\cdot \bm{x})\neq \sign(\bm{v}\cdot \bm{x})) $.

\section{Algorithms}\label{section:algorithms}

We study three algorithms: the {\em basic lifelong learning} algorithm (basic LLL) in Section \ref{section: basic_lll}, {\em lifelong learning with representation refinement} algorithm (LLL-RR) in Section \ref{section:algorithm_refinement}, and {\em heuristic lifelong learning} algorithm (H-LLL) in Section~\ref{section:HLLL}. We prove guarantees for basic LLL and LLL-RR in Section \ref{section:theoretical_guarantees}. We show that H-LLL for deep neural networks (Section~\ref{section:HLLL}) outperforms state-of-the-art continual learning algorithms in Section~\ref{section:real_data_exp}.

\subsection{Basic Lifelong Learning}\label{section: basic_lll}

Our algorithm maintains a set of features $\sigma_1(.), \ldots, \sigma_{r}(.)$ while tasks are presented incrementally. As is shown in Figure \ref{fig:sketch_lll}, when the next task, say $(i+1)$-th task arrives, the algorithm first tries to learn a new linear combination $y_{i+1}$ of existing features using examples from the current task. If the best such combination has a low error, it records the linear combination parameters and moves on to the next task. If the error is higher than a threshold $\epsilon$, then it learns a new set of features $\sigma_{r+1},\cdots,\sigma_{r+k_0}$ and a new linear combination of them with error up to $\epsilon_{acc}$. Denote $\tilde{k}$ as the number of steps that the algorithm learns new features. Let $k_0$ be the number of features learned at one time, it is a constant dependent on whether the features are linear or not. We describe the algorithm in Algorithm~\ref{algorithm:basic_lll}.

\begin{algorithm}[htp]
    \caption{Basic Lifelong learning Algorithm (Basic LLL)}
    \label{algorithm:basic_lll}
\KwIn{$d,m,k$, labeled examples of $m$ tasks, threshold parameters $\epsilon_{acc}, \epsilon$.}
\begin{algorithmic}
\State The algorithm maintains a set of features $\sigma_1(.), \ldots, \sigma_{r}(.)$ along training. When task $i+1$ arrives,
\end{algorithmic}
\vspace{-0.2in}
 \begin{itemize}
    \item Use the data from the $(i+1)$-th task, attempt to learn the linear function $\tilde{\bm{c}}_{i+1}$ using the current features $\bm{\sigma}(\cdot) = (\sigma_1(\cdot),\cdots,\sigma_{r}(\cdot))^\top$. 
    \vspace{-0.1in}
    \item Check whether the hypothesis $\bm{x} \mapsto\sign( \tilde{\bm{c}}_{i+1}^\top\bm{\sigma}(\bm{x}))$ has error less than $\epsilon$.
    \vspace{-0.15in}
    \begin{enumerate}
        \item If yes, record the linear combination parameters $\tilde{\bm{c}}_{i+1}$. 
        \vspace{-0.05in}
        \item Otherwise, learn a new set of features $\bm{\sigma}'(\cdot)=(\sigma_{r+1}(\cdot), \cdots, \sigma_{r+k_0}(\cdot))^\top$ and a linear function $\tilde{\bm{c}}_{i+1}$ such that the predictor $\bm{x}\mapsto \sign(\tilde{\bm{c}}_{i+1}^\top \bm{\sigma}'(\bm{x}))$ has error less than $\epsilon_{acc}$. \\
        Update the representation $\bm{\sigma}(\cdot) = (\sigma_1(\cdot),\cdots,\sigma_{r+k_0}(\cdot))^\top$.
    \end{enumerate}
\end{itemize}
\vspace{-0.15in}
\Return{$m$ predictors: $\bm{x} \xrightarrow{} \sign(\Tilde{\bm{c}}_i^\top\bm{\sigma}(\bm{x}))$, where $\bm{\sigma}(\bm{x}) = (\sigma_1(\bm{x}),\cdots,\sigma_{\tilde{k}k_0}(\bm{x}))^\top,1\leq i\leq m$.}
\end{algorithm}
\vspace{-0.1in}

The algorithm works for both linear and nonlinear features. For linear features, if a new target function does not have a good representation as a combination of the features learned so far, the new target is itself a new feature since everything is linear (Algorithm 1, \cite{balcan2015efficient}), so $k_0$ above is $1$. 
For nonlinear features, when the current representation is not good enough, we can learn a set of $k_0 \le k$ nonlinear features with low error since each task corresponds to a target with at most $k$ features. Here we assume that a single such combination can be learned efficiently (\ie, a neural network with a small, single hidden layer) \cite{bartlett2019nearly}. Section~\ref{section:theoretical_guarantees} proves that the number of features to be learned can be upper bounded by $\tilde{O}(kk_0)$. We give the full guarantees for this basic LLL algorithm in Theorem~\ref{thm:vanilla_algo}.

\subsection{Lifelong Learning with Representation Refinement}\label{section:algorithm_refinement}

Similar to the basic LLL algorithm, LLL-RR also expands the feature space gradually. Whenever we learn a new task $i$, we attempt to learn it using the current representation and check whether a linear combination exists with an error less than $\epsilon$. If yes, we record the classifier for the current task and move to the next one.  Otherwise, we learn a new classifier for the current task with error at most $\epsilon_{acc}$, via new features; we then do a step of \textit{representation refinement} on all the features learned so far. The refinement step can also be done when the number of features grows above a threshold rather than every time a new task is learned to high accuracy. The formal description of LLL-RR is given in Appendix~\ref{section:appendix_algo_lllrr}.

\paragraph{Refinement algorithm.} Denote $\tilde{\bm{w}}_1,\cdots, \tilde{\bm{w}}_{(\hat{k}+1)k_0}$ as all the features learned so far. The goal of refinement is to find a minimal dimensional feature subspace that is within distance $\epsilon_{acc}$ to all learned features. We \textit{minimize the dimension} of feature subspace while keeping it close to the original representation by solving the optimization problem (\ref{equation:original_opt}). This problem is NP-hard, but we provide an efficient approximation algorithm for the linear case (and practical implementation for the general case in Section~\ref{section:HLLL}).

\begin{algorithm}[htp]
    \caption{Representation Refinement (RR)}
     \label{algorithm:refinement}
\KwIn{All features learned so far $ \tilde{\bm{w}}_1,\cdots,\tilde{\bm{w}}_{(\hat{k}+1)k_0}$, and the desired feature subspace dimension $k$.}
\begin{algorithmic}
    \State Solve the following optimization problem, and get the solution $\bm{V}'$.
\end{algorithmic}
\begin{equation}
\label{equation:original_opt}
    \begin{aligned}
    \min_{\bm{V}} \quad &\text{dim} \left(\bm{V}\right)\\
    \text{s.t.}\quad&d\left(\tilde{\bm{w}}_i, \bm{V}\right)\leq \epsilon_{acc},\quad \text{for }1\leq i \leq \left(\hat{k}+1\right)k_0 \\
\end{aligned}
\end{equation}
\Return{Refined representation $\bm{V}'$.}
\end{algorithm}

The refinement step is provably beneficial to the total sample complexity. Theorem \ref{thm:tune_linear} guarantees that lifelong learning algorithm with representation refinement (LLL-RR) can be ended with learning new features in $\tilde{O}(k)$ steps. The analysis is shown in Section \ref{section:theoretical_guarantees}.

\paragraph{Linear features.}
We provide an efficient implementation for the linear case by using a Semi-definite Programming (SDP) relaxation (\ref{equation:sdp_relation}) and then applying Principal Component Analysis (PCA) to round the SDP solution. The relaxation from (\ref{equation:original_opt}) to (\ref{equation:sdp_relation}) is natural. The positive semi-definite (PSD) matrix $\bm{X}$ represents the projection matrix to $\bm{V}^\bot$, the complement of the subspace $\bm{V}$. It is a relaxation since $\bm{X}$ might have fractional eigenvalues between $0$ and $1$. We describe the formal algorithm in Algorithm~\ref{algorithm:linear_refinement}. As is proved in Theorem \ref{thm:sdp_approx} (Section~\ref{section:theoretical_guarantees}), if the optimal dimension of the feature subspace is $k$, this linear case implementation will output a $(2k-1)$-dimensional subspace $\bm{V}'$ with $d(\tilde{\bm{w}}_i,\bm{V}')\leq \sqrt{2}\epsilon_{acc}, \forall i\in[\hat{k}+1]$. Consequently, LLL-RR terminates with feature dimension $O(k)$.

\begin{algorithm}[htp]
    \caption{Representation Refinement (RR) Implementation in Linear Case}
    \label{algorithm:linear_refinement}
    \KwIn{All features learned so far $ \tilde{\bm{w}}_1,\cdots,\tilde{\bm{w}}_{\hat{k}+1}$, and the desired feature subspace dimension $k$.}
\begin{enumerate}
\item Solve the following SDP, and get the solution $\bm{X}^*, t^*$.
\begin{equation}
\label{equation:sdp_relation}
    \begin{aligned}
    \min_{\bm{X},t}\quad & t\\
    \textrm{s.t.} \quad & \tilde{\bm{w}}_i^\top \bm{X}\tilde{\bm{w}}_i \leq t, 1\leq i\leq \hat{k}+1\\
    &0\preceq \bm{X}\preceq I\\
    &\tr{(\bm{X})}=d-k
    \end{aligned}
    \vspace{-0.15in}
\end{equation}
\item Do the singular value decomposition $\bm{X}^*=\sum_{i=1}^d\lambda_i\bm{u}_i\bm{u}_i^\top$, where $0\leq \lambda_1\leq \cdots\leq \lambda_d \leq 1$.
\end{enumerate}
\Return{Refined representation $\bm{V}'=\lspan(\bm{u}_1,\cdots,\bm{u}_{2k-1})$.}
\end{algorithm}
\vspace{-0.1in}

\subsection{A Lifelong Learning Heuristic for Deep Neural Networks}
\label{section:HLLL}

In order to apply our basic LLL algorithm to deep neural networks, we propose a heuristic lifelong learning (H-LLL) algorithm. The intuition of our LLL algorithm is to build an expandable and dynamic representation that can adapt to incoming tasks/classes without sacrificing the quality for previous tasks/classes. Following this intuition, we propose to learn a separate encoder for each task. We observe the training data $\mathcal{D}_i$ for the $i$-th task and the memory buffer $\mathcal{M}_i$ for the previous tasks. The memory buffer is constructed based on herding selection~\cite{welling2009herding,rebuffi2017icarl}. H-LLL works iteratively in two phases. First, H-LLL learns the representation with a separate encoder $f_{i}$ in the $i$-th task, while the other encoders $f_j,j<i$ are frozen during the training in the $i$-th task. Second, H-LLL finetunes the last classifier layer using the memory buffer $\mathcal{M}_i$ and the current task data $\mathcal{D}_i$. These two steps are iterated as the training proceeds. We take the $i$-th task as an example. Since we train a separate encoder $f_i$ for the $i$-th task, the representation of a sample $\bm{x}$ (by the end of the $i$-th task) is constructed by concatenating all the learned features: $\bm{v}_i(\bm{x})=\{f_1(\bm{x}),f_2(\bm{x}),\cdots,f_i(\bm{x})\}$ where $\bm{v}_i$ denotes the representation after learning the $i$-th task. The training uses cross-entropy loss on both the memory buffer $\mathcal{M}_i$ and the current dataset $\mathcal{D}_i$:
\begin{equation}
    \mathcal{L}=-\frac{1}{\left| \mathcal{M}_i \cup \mathcal{D}_i\right|} \sum_{i=1}^{\left| \mathcal{M}_i \cup \mathcal{D}_i\right|}\log\left(\textnormal{SoftMax}\left( \bm{W}_{\textnormal{cls}}^\top \bm{v}_i(\bm{x}) \right)\right)
\end{equation}
where $\bm{W}_{\textnormal{cls}}$ is the weight of the last classifier layer. After training of the representation is completed, we follow \cite{yan2021dynamically} and re-train the classifier layer with a heated-up softmax~\cite{zhang2018heated} and a balanced finetuning method~\cite{castro2018end}. Note that, for each encoder $f_j,\forall j$, we can parameterize it with any neural network. In this paper, we use ResNet-18 as $f_j,\forall j$.

\section{Theoretical Guarantees}\label{section:theoretical_guarantees}

Here we state the main theorems for the basic LLL algorithm and LLL-RR algorithm, bounding the representation size and complexity. Here our algorithm and analysis apply for both linear and nonlinear features. For nonlinear features, we consider the kernel induced by them. These features live in a potentially infinite-dimensional space (or exponential in $d$ dimensional space if, \eg, the input is from the Boolean hypercube).

The theorem for the basic lifelong learning algorithm is stated as follows.

\begin{thm}[Basic LLL]\label{thm:vanilla_algo} 
Consider the lifelong learning setting of input dimension $d$, $m$ tasks with $k$ common features.
Let $\epsilon_{acc}=\frac{\epsilon}{c\sqrt{k}}$ for a sufficiently small constant 
$c >0$. Under Assumption \ref{equation:assumption}, the basic LLL algorithm, learns new features at most $\tilde{k}=O(k\log(\log(k)/\epsilon))$ times and the dimension of the learned feature space is $O(k\log(\log(k)/\epsilon))$ for linear features and $O(k^2\log(\log(k)/\epsilon))$ for nonlinear features. The total number of labeled examples to learn all tasks to within  error $\epsilon$ is 
$O(\frac{dk^{1.5}}{\epsilon}\log(\frac{\log(k)}{\epsilon})\log(\frac{k}{\epsilon})+\frac{km}{\epsilon}\log(\frac{\log(k)}{\epsilon})\log(\frac{1}{\epsilon}))=\tilde{O}(dk^{1.5}/\epsilon+km/\epsilon)$ for linear features and a factor of $k$ higher for nonlinear features.
\end{thm}

Our main result analyzes the lifelong learning algorithm with representation refinement. 

\begin{thm}[LLL with Representation Refinement]\label{thm:tune_linear} Consider the lifelong learning setting of input dimension $d$, $m$ tasks with $k$ common features. Suppose that the algorithm has access to an oracle that gives a constant-factor approximation of Optimization Problem \ref{equation:original_opt}. Set
$\epsilon_{acc}=\frac{\epsilon}{c\sqrt{k}}$ for a sufficiently small constant $c>0$. Under Assumption \ref{equation:assumption}, the LLL-RR algorithm learns at most $O(k\log(\log(k)/\epsilon))$ new features, and the dimension of the feature space is $O(k)$. The total number of labeled examples to learn tasks to within error $\epsilon$ is $O(\frac{dk^{1.5}}{\epsilon}\log(\frac{\log(k)}{\epsilon})\log(\frac{k}{\epsilon})+\frac{km}{\epsilon}\log(\frac{1}{\epsilon}))=\tilde{O}(dk^{1.5}/\epsilon+km/\epsilon)$.
\end{thm}

In the linear setting, we provide an efficient implementation of the constant-factor approximation oracle in Algorithm~\ref{algorithm:linear_refinement} with the following guarantee.

\begin{thm}[Approximation]\label{thm:sdp_approx}
In the linear case, for the optimization problem (\ref{equation:original_opt}), if there exists a subspace $\bm{V}^*$ of $k$ dimension with $d(\tilde{\bm{w}}_i, \bm{V}^*) \leq \epsilon_{acc}, \forall i\in[\hat{k}+1]$, then for any constant $c>1$, we can get a $(ck-1)$-dimensional subspace solution with approximation factor $\sqrt{1+\frac{1}{c-1}}$ in maximum distance. Specifically, let $c=2$, the output of Algorithm \ref{algorithm:linear_refinement}, $\bm{V}'$, is a $(2k-1)$-dimensional subspace s.t. $d(\tilde{\bm{w}}_i, \bm{V}')\leq \sqrt{2}\eps_{acc},\forall i\in[\hat{k}+1]$.
\end{thm}

In Appendix~\ref{section:appendix_theoretical_basic}, we give
another approach for analyzing the basic LLL algorithm with weaker guarantees, similar to the proof in \cite{balcan2015efficient}. 

\subsection{Proof of Guarantees for Basic LLL algorithm and LLL-RR algorithm}\label{section:lllrr}

We will prove the guarantees for LLL-RR algorithm (Theorem~\ref{thm:tune_linear}) first, and then the proof of Theorem~\ref{thm:vanilla_algo} follows easily. 

\paragraph{Proof idea and plan.}
Our proof is based on geometry. Firstly, we show that Assumption~\ref{equation:assumption} guarantees that the distance between hypothesis vectors approximates the test error within a constant factor. 
As is shown in Figure~\ref{fig:intuition}, for any target $\bm{a}_i$ that has a large error based on the previous features, $d(\bm{a}_i, \bm{V}_{i-1}) \geq \eps$ neglecting constants. Learning $\bm{a}_i$ accurately will help reduce the angle between the feature subspace of the algorithm $\bm{V}_i$ and the true feature subspace $\bm{V}^*$. To quantify the improvement in the angle between the feature subspaces, we construct a convex set whose volume grows at a geometric rate. This leads to the upper bound on the number of new features.

To be more specific, denote $i_1,\cdots,i_{\tilde{k}}$ as the indices of tasks where we learn new features. At step $i_{\hat{k}}$, we construct a set $Y_{i_{\hat{k}}}$ of all possible subspaces that are feasible solutions to the refinement optimization problem (\ref{equation:original_opt}). Let $X_{i_{\hat{k}}}$ be the set of vectors in the unit ball in the true $k$-dimensional feature subspace that are within distance $O(\eps/\sqrt{k})$ to all subspaces in $Y_{i_{\hat{k}}}$. Then we can show that $X_{i_{\hat{k}}}$ is a symmetric convex set. Clearly, the set $Y_{i_{\hat{k}}}$ shrinks during training as we have more and more constraints in the optimization problem. Alongside, the volume of $X_i$ increases exponentially. The learning procedure terminates when $X_{i_{\hat{k}}}$ covers the ball $B_k(0,1/2\sqrt{k})$, which means that a target function (unit vector) spanned by the true $k$ features will have error $O(\epsilon)$ to the solution learned by LLL-RR. In other words, the feature subspace we learn can solve all future tasks with small errors using only hypothesis vectors from the learned feature subspace.

\begin{figure}[htbp]
    \centering
    \includegraphics[height=2.3in]{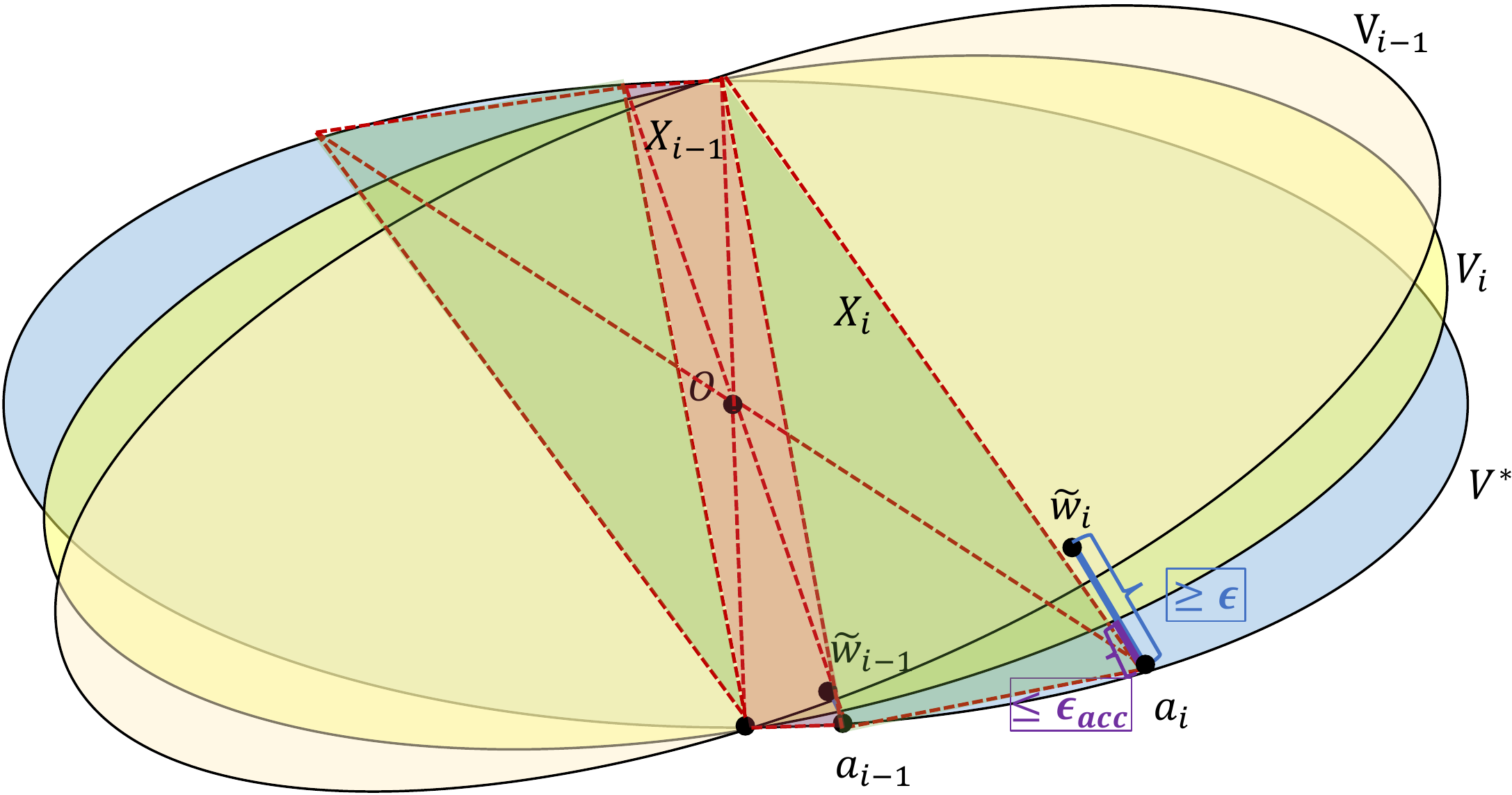}
    \caption{Geometric illustration of the proof sketch. For any target $\bm{a}_i$ that has more than $\eps$ error based on the previous feature subspace $\bm{V}_{i-1}$, the algorithm accurately learns a new feature within error $\eps_{acc}$, and therefore pushes the new feature subspace $\bm{V}_i$ towards the true one $\bm{V}^*$.}
    \label{fig:intuition}
\end{figure}

We will prove this step by step. Lemma~\ref{lemma:error2dist} bridges the test error to the distance metric. Lemma~\ref{lemma:cvx_set}, Lemma~\ref{lemma:in_set} and Corollary~\ref{cor:in_set} show that the convex hull of true feature vectors is contained in the set $X_{i_{\hat{k}}}$. Lemma~\ref{lemma: max_ellipsoid} carefully analyzes the maximum volume ellipsoid in $X_{i_{\hat{k}}}$, whose volume grows exponentially. Based on these facts, we can bound the number of new features and prove Theorem~\ref{thm:tune_linear}.

We will first prove the theorem for the linear setting. Then it can be naturally extended to the nonlinear case when we consider the kernel of the features, regarding each feature as a potentially infinite-dimensional vector. 

Let $\bm{X}=\mathbb{R}^d$, for each task $i$, there exists unit length $\bm{a}_i$ such that all $(\bm{x},y)$ drawn from $P$ satisfies $\text{sign}(\langle\bm{a}_i, \bm{x}\rangle)=y$. Let $\bm{A} \in\mathbb{R}^{m\times d}$, rows of which are $\bm{a}_i^\top$. Since the parameters $\bm{a}_i$ lie in some $k$-dimensional subspace with $k\ll\min(m,n)$, there exists $\bm{W}\in\mathbb{R}^{k\times d}, \bm{C}\in\mathbb{R}^{m\times k}$ such that $\bm{A}=\bm{CW}$. Rows $\bm{w}_1^\top,\cdots,\bm{w}_k^\top$ can be seen as $k$ linear meta-features that are sufficient to learn $m$ tasks. In each step when the current feature subspace cannot achieve low error, we learn new features. Then we take the refinement step to keep a minimal dimensional subspace that is close to all current features $\tilde{\bm{w}}_1,\cdots, \tilde{\bm{w}}_{\hat{k}}$.

\begin{lemma}\label{lemma:error2dist}
Given two unit vectors $\bm{u}, \bm{v}$ and a distribution $D$. If $D$ satisfies Assumption \ref{equation:assumption}, then there exist nonzero constants $c'$ and $c''$ such that $c'\|\bm{u}-\bm{v}\|_2\leq d_D(\bm{u},\bm{v}) \leq c''\|\bm{u}-\bm{v}\|_2$.
\end{lemma}
\begin{proof}
    By Assumption \ref{equation:assumption}, $\exists c_1,c_2$ s.t. $c_1\theta(\bm{u},\bm{v})\leq d_D(\bm{u},\bm{v})\leq c_2\theta(\bm{u},\bm{v})$. Using the Taylor expansion of cosine function, we know $1-x^2/2 \leq \cos(x) \leq 1-x^2/2!+x^4/4! \leq 1 - 11x^2/24$. Since $\|\bm{u}-\bm{v}\|_2^2 = 2-2\cos(\theta(\bm{u},\bm{v}))$, we have $\sqrt{\frac{11}{12}}\theta(\bm{u},\bm{v}) \leq \|\bm{u}-\bm{v}\|_2 \leq \theta(\bm{u},\bm{v})$. Choose $c'=c_1,c''=\sqrt{\frac{12}{11}}c_2$, we get the results proved. 
\end{proof}

\begin{lemma}\label{lemma:cvx_set}
Let $S$ be a set of subspaces. Let $X=\{\bm{x}\in B_k(0,1)| d(\bm{x}, \bm{Y})\leq r, \forall \bm{Y}\in S\}$. Here $B_k(0,1)$ is the unit ball in $k$-dimension. Then the set $X$ is a symmetric convex set.
\end{lemma}

\begin{proof}
For any $\bm{x}\in X, \forall \bm{Y}\in S$, since $d(\bm{x},\bm{Y})=d(-\bm{x}, \bm{Y})$, we have $-\bm{x} \in X$. So $S$ is symmetric about the origin.

For any $\bm{x}_1,\bm{x}_2\in X$, for any $\bm{Y}\in S$, we have $d(\bm{x}_1, \bm{Y})\leq r, d(\bm{x}_2,\bm{Y})\leq r$. Let $\bm{P}$ be the projection matrix of $\bm{Y}$, 
\begin{align*}
    &\left(\|\bm{P}\bm{x}_1\|_2 + \|\bm{P}\bm{x}_2\|_2\right)^2 - \|\bm{P}\left(\bm{x}_1+\bm{x}_2\right)\|_2^2\\
    =& \bm{x}_1^\top \bm{P}^\top \bm{P}\bm{x}_1+\bm{x}_2^\top \bm{P}^\top \bm{P}\bm{x}_2+2\|\bm{P}\bm{x}_1\|_2\|\bm{P}\bm{x}_2\|_2-\left(
    \bm{x}_1+\bm{x}_2\right)^\top \bm{P}^\top \bm{P}\left(\bm{x}_1+\bm{x}_2\right)\\
    =& \|\bm{P}\bm{x}_1\|_2\|\bm{P}\bm{x}_2\|_2-\left(\bm{P}\bm{x}_1\right)^\top\left(\bm{P}\bm{x}_2\right)
    \geq  0
\end{align*}
So we have 
\[
d\left(\frac{\bm{x}_1+\bm{x}_2}{2}, \bm{Y}\right) 
= \left\|\bm{P}\left(\frac{\bm{x}_1+\bm{x}_2}{2}\right)\right\|_2
\leq \left\|\bm{P}\left(\frac{\bm{x}_1}{2}\right)\right\|_2+\left\|\bm{P}\left(\frac{\bm{x}_2}{2}\right)\right\|_2 
= \frac{1}{2}d\left(\bm{x}_1,\bm{Y}\right) +  \frac{1}{2}d\left(\bm{x}_2,\bm{Y}\right)\leq r
\]
For a fixed $\bm{Y}$, $\{\bm{x}\in B_k(0,1) | d(\bm{x},\bm{Y})\leq r\}$ is closed and thus convex. Therefore 
\[
X=\bigcap_{\bm{Y}\in S} \{\bm{x}\in B_k(0,1) | d(\bm{x},\bm{Y})\leq r\}
\]
is a convex set.
\end{proof}

\begin{lemma}\label{lemma:in_set} For any $\hat{k}\in [\tilde{k}]$, let 
\[
Y_{i_{\hat{k}}} := \{\bm{V} | d(\tilde{\bm{w}}_{i_j}, \bm{V}) \leq c_1\epsilon_{acc}, \forall j\leq \hat{k}\},
\]
\[
X_{i_{\hat{k}}} := \{\bm{x}\in B_k(0,1) | d(\bm{x}, \bm{V}) \leq (c_1+\frac{1}{c'})\epsilon_{acc}, \forall \bm{V}\in Y_{i_{\hat{k}}}\},
\]
where $c_1, c'$ are small constants. Then for any $j\geq \hat{k},  \pm \bm{a}_{i_{\hat{k}}} \in X_{i_{j}}$.
\end{lemma}
\begin{proof}
For $\forall \bm{V}\in Y_{i_{\hat{k}}}$, $d(\tilde{\bm{w}}_{i_{\hat{k}}},\bm{V})\leq c_1\epsilon_{acc}$. Since we learn the feature vector $\tilde{\bm{w}}_{i_{\hat{k}}}$ within error $\epsilon_{acc}$, by Lemma \ref{lemma:error2dist}, we have $d(\bm{a}_{i_{\hat{k}}},\tilde{\bm{w}}_{i_{\hat{k}}})\leq \epsilon_{acc}/c'$. So  $d(\bm{a}_{i_{\hat{k}}},\bm{V}) \leq (c_1+\frac{1}{c'})\epsilon_{acc}$. Hence $\bm{a}_{i_{\hat{k}}}\in X_i$. Because $d(\bm{x},\bm{V})=d(-\bm{x}, \bm{V})$, we also know $-\bm{a}_{i_{\hat{k}}}\in X_i$.
Also $Y_{i_1}\supseteq Y_{i_2}\supseteq \cdots \supseteq Y_{i_{\tilde{k}}}$, so $X_{i_1}\subseteq X_{i_2}\subseteq \cdots \subseteq X_{i_{\tilde{k}}}$. So for any $j \geq \hat{k}$, $\pm \bm{a}_{i_{\hat{k}}} \in X_{i_j}$.
\end{proof}

\begin{corollary}\label{cor:in_set}
For any $\hat{k}\in [\tilde{k}]$, $\conv(\pm \bm{a}_{i_1},\cdots,\pm \bm{a}_{i_{\hat{k}}}) \subseteq X_{i_{\hat{k}}}$.
\end{corollary}
\begin{proof}
From Lemma \ref{lemma:in_set}, we know $\pm \bm{a}_{i_1},\cdots,\pm \bm{a}_{i_{\hat{k}}} \in X_{i_{\hat{k}}}$. Combined with Lemma \ref{lemma:cvx_set}, we get the corollary.
\end{proof}

\begin{lemma}[Max Ellipsoid]\label{lemma: max_ellipsoid}
Let $K \subset \R^k$ be a symmetric convex body and $E(K)$ be the maximum volume ellipsoid contained in $K$. 
For a vector $\bm{u}$ on the boundary of $E(K)$, \ie $\bm{u}\in\bd({E(K)})$, let $K'=\conv(K, 2\sqrt{k}\bm{u}, -2\sqrt{k}\bm{u})$. Then, 
\[
\frac{\vol\left(E\left(K'\right)\right)}{\vol\left(E\left(K\right)\right)} \geq \frac{13}{10}.
\]
\end{lemma}

\begin{proof}
Since $E(K) \subseteq K$, we have
\[
K'':=\conv\left(E\left(K\right), 2\sqrt{k}\bm{u},-2\sqrt{k}\bm{u}\right) \subseteq K'
\]
So it suffices to prove that
\[
\frac{\mbox{vol}\left(E\left(K''\right)\right)}{\mbox{vol}\left(E\left(K\right)\right)} \geq \frac{13}{10}
\]
First, let's assume that $E(K)=B_k(0,1)$, \ie, the unit ball around the origin, and the vector $\bm{u}$ is $(1,0,\cdots,0)^\top \in B_k(0,1)$. By symmetry, we can assume the ellipsoid
\[
E(K'')=\left\{\bm{x}\left\vert\frac{x_1^2}{a^2}+\sum_{i=2}^k \frac{x_i^2}{b^2} \leq 1\right.\right\}
\]
We consider the two-dimensional slice first, where $\bm{u}=(1,0)^\top$, $E(K'')=\{(x,y)\vert\frac{x^2}{a^2}+\frac{y^2}{b^2}\leq 1\}$. Direct calculation shows that 
$$\forall (x,y)\in K'', y\leq- \frac{x}{\sqrt{4k-1}}+\frac{2\sqrt{k}}{\sqrt{4k-1}}$$ where $y=- \frac{x}{\sqrt{4k-1}}+\frac{2\sqrt{k}}{\sqrt{4k-1}}$ is the line tangent to the unit ball and go across the point $(\frac{1}{2\sqrt{k}}, \frac{\sqrt{4k-1}}{2\sqrt{k}})$ and the point $(2\sqrt{k},0)$ (see Figure \ref{fig:ellipsoid}). 
\begin{figure}[htbp]
    \centering
    \includegraphics[height=1.8in]{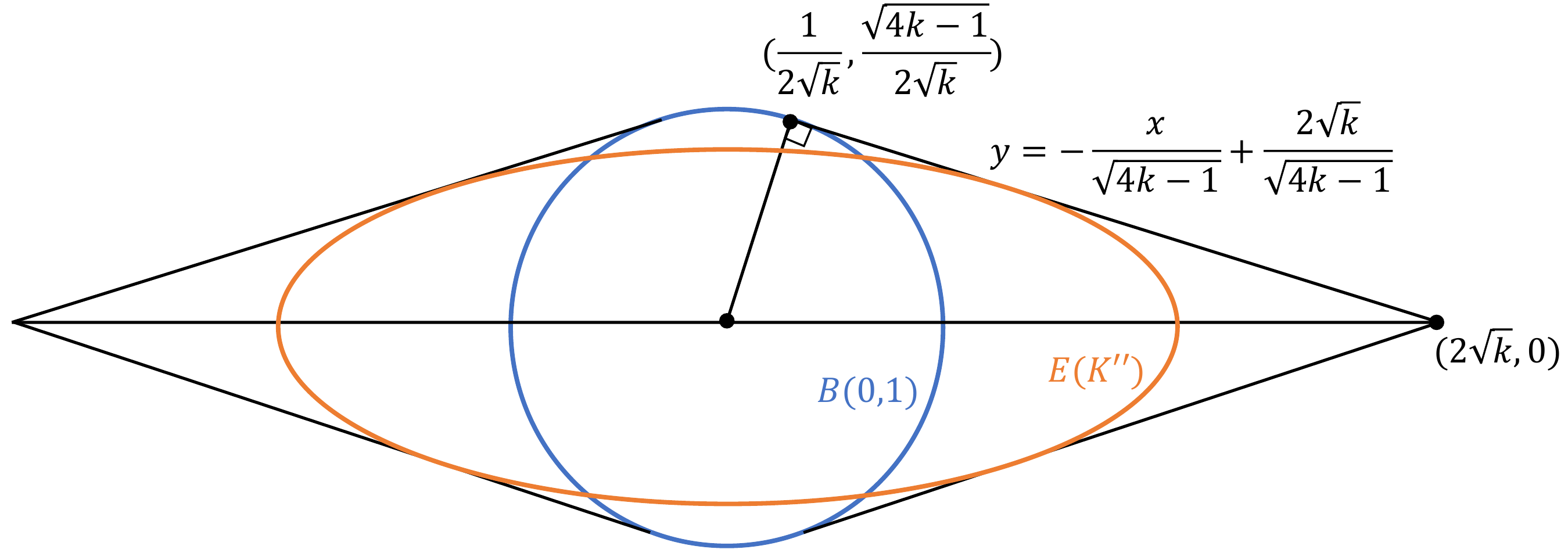}
    \caption{Maximum volume ellipsoid}
    \label{fig:ellipsoid}
\end{figure}
So for any point on the boundary of $E(K'')$, it also satisfies $$y^2=b^2\left(1-\frac{x^2}{a^2}\right)\leq\left(- \frac{x}{\sqrt{4k-1}}+\frac{2\sqrt{k}}{\sqrt{4k-1}}\right)^2$$
Simplifying the inequality we get
\begin{align*}
    \left(\frac{b^2}{a^2}+\frac{1}{4k-1}\right)x^2-\frac{4\sqrt{k}}{4k-1}x+\frac{4k}{4k-1}-b^2\geq 0
\end{align*}
To ensure that the quadratic inequality holds, let the determinant equal zero, and we get $a^2=4k-b^2(4k-1)$. So,
\[
\left\{\left(x,y\right)\left\vert\frac{x^2}{4k-b^2(4k-1)}+\frac{y^2}{b^2}\leq 1\right.\right\} \subseteq K'' \mbox{ for } b<1. 
\]
By symmetry, in $k$ dimensions, we get that $E_b \subseteq E(K'')$ for $b<1$.
\begin{align*}
    E_b=\left\{x \left\vert \frac{x_1^2}{4k-b^2(4k-1)}+\sum_{i=2}^k \frac{x_i^2}{b^2} \leq 1\right.\right\}
\end{align*}
The volume of this ellipsoid is
\begin{align*}
    \mbox{vol}\left(E_b\right)=\sqrt{\left(4k-b^2\left(4k-1\right)\right)b^{2k-2}}\cdot \mbox{vol}\left(B_k\left(0,1\right)\right)
\end{align*}
Let $f(b)=(4k-b^2(4k-1))b^{2k-2}$. Calculate its derivative and let it to be zero, so we get $\hat{b}^2=1-\frac{3}{4k-1}<1$. 
$$f\left(\hat{b}\right)=4\left(1-\frac{3}{4k-1}\right)^{k-1}=4\left(\left(1-\frac{3}{4k-1}\right)^{\frac{4k-4}{3}}\right)^{\frac{3}{4}} \geq 4\left(\frac{1}{e}\right)^{\frac{3}{4}}$$
Hence we know 
$$\mbox{vol}\left(E_{\hat{b}}\right) \geq 2\left(\frac{1}{e}\right)^{\frac{3}{8}}\mbox{vol}\left(B_k\left(0,1\right)\right) \geq \frac{13}{10}\mbox{vol}\left(B_k\left(0,1\right)\right)$$
Finally for any symmetric ellipsoid $E(K):=\bm{A}^{1/2}B_k(0,1)$ and any $\bm{u}=\bm{A}^{1/2}\bm{u}_0$ on the boundary of $E(K)$, where $\bm{u}_0$ is the unit vector corresponding to $\bm{u}$. There exists an orthogonal matrix, say $\bm{Q}$, that rotates $\bm{u}_0$ to $(1,0,\cdots,0)^\top$. That is $\bm{Q}\bm{u}_0 = (1,0,\cdots,0)^\top$.
Define an affine bijective transformation $T:=\bm{A}^{1/2}\bm{Q}^\top \bm{x}$, with $T^{-1}(\bm{x})=\bm{QA}^{-1/2}\bm{x}$. Then 
\begin{align*}
    T\left(\left(1,0,\cdots,0\right)^\top\right)=\bm{A}^{1/2}\bm{Q}^\top \bm{Q}\bm{u}_0=\bm{A}^{1/2}\bm{u}_0=\bm{u}
\end{align*}
\begin{align*}
    T\left(B_k\left(0,1\right)\right)=&\{T\left(\bm{y}\right)|\bm{y}^\top \bm{y} \leq 1\}\\
    =& \{\bm{x} | T^{-1}\left(\bm{x}\right)^\top T^{-1}\left(\bm{x}\right) \leq 1\}\\
    =&\{\bm{x} | \bm{x}^\top \bm{A}^{-1/2}\bm{Q}^\top \bm{QA}^{-1/2}\bm{x} \leq 1\}\\
    =& \{\bm{x} | \bm{x}^\top \bm{A}^{-1}\bm{x} \leq 1\}\\
    =& E\left(K\right)
\end{align*}
So we get $T(E_{\hat{b}}) \subset E(K'')$. Since the ratio of volumes is invariant under affine transformation, we have
\begin{align*}
    \frac{\mbox{vol}\left(E\left(K''\right)\right)}{\mbox{vol}\left(E\left(K\right)\right)} 
    \geq \frac{\mbox{vol}\left(T\left(E_{\hat{b}}\right)\right)}{\mbox{vol}\left(T\left(B_k\left(0,1\right)\right)\right)}
    = \frac{\mbox{vol}\left(E_{\hat{b}}\right)}{\mbox{vol}\left(B_k\left(0,1\right)\right)}
    \geq \frac{13}{10}.
\end{align*}

\end{proof}

The next lemma gives us a stopping condition.

\begin{lemma}\label{lem:maxvol}
Let $P=\conv(\bm{x}_1,\ldots,\bm{x}_m)$ be a polytope in $\R^k$ with each $\bm{x}_i$ of unit Euclidean length. 
Then, the maximum volume ellipsoid contained in $P$ satisfies
\[
\vol\left(E\left(P\right)\right)\le 2\sqrt{2e}\left(\sqrt{\frac{2\log 2m}{k}}\right)^k\vol\left(B_k\left(0,1\right)\right). 
\]
\end{lemma}

\begin{proof}
Recall the polar of a convex body $P$ is the convex body defined as \[
P^*=\{\bm{x}\,:\,\langle \bm{x},\bm{y}\rangle \le 1 \mbox{ for all } \bm{y}\in P\}.
\]
By the Blaschke-Santalo inequality, we have
\[
\vol\left(P\right)\vol\left(P^*\right)\le \vol\left(B_k\left(0,1\right)\right)^2.
\]
Next we lower bound the volume of $P^*$. Note that $P^*$ is the intersection of exactly $m$ halfspaces, each tangent to the unit ball. Consider the ball $B_k(0,r)$ with $r=\sqrt{\frac{k-1}{2\log(2m)}}$. By Lemma~\ref{lem:capvol}, each halfpace cuts off a cap of this ball, of volume at most $e^{-(k-1)/2r^2}=\frac{1}{2m}$ of the volume of $B_k(0,r)$. Therefore, the volume that is in the intersection of all $m$ halfspaces is at least $\vol(B_k(0,r))/2$ and hence, this is a lower bound on the volume of $P^*$.  
Using this, we have,
\[
\vol\left(P\right) \le \frac{\vol\left(B_k\left(0,1\right)\right)^2}{\vol\left(B_k\left(0,r\right)\right)/2} \le 2r^{-k} \vol\left(B_k\left(0,1\right)\right)
\]
Furthermore, since $1+x\leq e^x$, we can derive the following and complete the proof.
\[
2r^{-k}=2(\sqrt{\frac{2\log(2m)}{k-1}})^k
=2(\sqrt{\frac{2\log(2m)}{k}})^k(1+\frac{1}{k-1})^{\frac{k-1}{2}}\sqrt{1+\frac{1}{k-1}}
\leq 2\sqrt{2e}(\sqrt{\frac{2\log(2m)}{k}})^k
\]

\end{proof}

\begin{lemma}\label{lem:capvol}
(Lemma 4.1 from  \cite{lovasz2007geometry})For any $\frac{1}{\sqrt{k}}<t<1$ and halfspace $H$ at distance $tr$ from the origin, 
\[
\vol\left(B_k\left(0,r\right) \cap H\right) \le e^{-t^2k/2} \vol\left(B_k\left(0,r\right)\right).
\]
\end{lemma}

We can now prove Theorem~\ref{thm:tune_linear}.

\begin{proof}
Let's first prove the linear case. Since the ground truth feature space lies in a $k$-dimensional subspace $\bm{V}^*=\lspan(\bm{a}_1,\cdots,\bm{a}_m) \subseteq \mathbb{R}^k$, the span of the true features is $\mathbb{R}^k$. Along training, when we deal with the $i_{\hat{k}}$-th task, let $\hat{\bm{w}}_{i_{\hat{k}}}$ be the new feature we learn to ensure that the $i_{\hat{k}}$-th task has error no more than $\epsilon_{acc}=\frac{\epsilon}{2\sqrt{k}c''(c_1+1/c')}$, where $c', c''$ are universal constants defined in Lemma~\ref{lemma:error2dist} and $c_1$ is the approximation constant of Optimization Problem~\ref{equation:original_opt}. WLOG we assume $\hat{\bm{w}}_{i_{\hat{k}}}$ to be a unit vector. By Lemma \ref{lemma:error2dist}, 
$$d(\bm{a}_{i_{\hat{k}}}, \hat{\bm{w}}_{i_{\hat{k}}}) \leq \frac{\epsilon_{acc}}{c'}=\frac{\epsilon}{2\sqrt{k}c''(c_1c'+1)}$$
Denote $\bm{V}_{i_{\hat{k}}}$ be the feature subspace after fine-tuning (optimization). Since there exists a $k$-dimensional subspace $\bm{V}^*$ satisfying all constraints and the algorithm outputs a constant-factor approximation of Optimization Problem~\ref{equation:original_opt}, we can get a $c_1$-approximation solution with dimension $c_2 k$ for constants $c_1,c_2$. So we know in the end the dimension of the feature subspace we get is $O(k)$.

Let $B_k(0,1)$ be the unit ball on the subspace. Denote the set of all possible solutions of the optimization as $Y_{i_{\hat{k}}} := \{\bm{V} | d(\tilde{\bm{w}}_{i_j}, \bm{V}) \leq c_1\eps_{acc}, \forall j\leq \hat{k}\}$. Let $X_{i_{\hat{k}}}$ be all the vectors in the unit ball that is within distance $(c_1+\frac{1}{c'})\eps_{acc}=\frac{\eps}{2\sqrt{k}c''}$ to all of the subspaces in $Y_{i_{\hat{k}}}$, that is $X_{i_{\hat{k}}} := \{\bm{x}\in B_k(0,1) | d(\bm{x}, \bm{V}) \leq \frac{\eps}{2\sqrt{k}c''}, \forall \bm{V}\in Y_{i_{\hat{k}}}\}$. By Lemma \ref{lemma:cvx_set} and Corollary \ref{cor:in_set}, we know $X_{i_{\hat{k}}}$ is a convex set containing $\{\pm \bm{a}_{i_1},\cdots,\pm \bm{a}_{i_{\hat{k}}}\}$. We will show next that after learning $\tilde{k}=O(k\log(\log(k)/\eps))$ new tasks, $X_{i_{\tilde{k}}}$ contains the ball $B_k(0,1/2\sqrt{k})$.

In the initial step, $Y_0$ is the set of all subspaces. We naturally have $B(0,\frac{\epsilon}{2\sqrt{k}c''}) \subseteq X_0$ since for any $\bm{x}\in B(0,\epsilon)$, $d(\bm{x}, \bm{V}) \leq d(\bm{O},\bm{V})+d(\bm{x},\bm{O})\leq \frac{\eps}{2\sqrt{k}c''}$. Here $\bm{O}$ is the origin. So we know $\vol(X_0)\geq (\frac{\epsilon}{2\sqrt{k}c''})^k V_0$, where $V_0$ is the volume of the unit ball in $\mathbb{R}^k$. 
Encountering the $i_{\hat{k}}$-th task, the current feature subspaces $\bm{V}_{i_{\hat{k}-1}}$ cannot ensure an $\epsilon$ error. By Lemma \ref{lemma:error2dist}, $d(\bm{a}_{i_{\hat{k}}}, \bm{V}_{i_{\hat{k}-1}}) \geq \frac{\epsilon}{c''} $. Hence $d(\frac{\bm{a}_{i_{\hat{k}}}}{2\sqrt{k}},\bm{V}_{i_{\hat{k}-1}}) \geq \frac{\epsilon}{2\sqrt{k}c''} $, which means $\pm \frac{\bm{a}_{i_{\hat{k}}}}{2\sqrt{k}} \notin X_{i_{\hat{k}-1}}$. Consequently, the vector $\bm{u}=u'\bm{a}_{i_{\hat{k}}} \in \bd(E(X_{i_{\hat{k}-1}}))$ satisfies $\|\bm{u}\| < \frac{1}{2\sqrt{k}}$. According to lemma \ref{lemma: max_ellipsoid}, we know that 
$$\mbox{vol}\left(\conv\left(X_{i_{\hat{k}-1}}, \pm \bm{a}_{i_{\hat{k}}}\right)\right)\geq \mbox{vol}\left(\conv \left(X_{i_{\hat{k}-1}}, \pm 2\sqrt{k}\bm{u}\right)\right) \geq  \frac{13}{10}\mbox{vol}\left(X_{i_{\hat{k}-1}}\right)$$
Also because of the convexity of $X_{i_{\hat{k}}}$, we have $X_{i_{\hat{k}}}\supseteq \conv(X_{i_{\hat{k}-1}}, \pm \bm{a}_{i_{\hat{k}}})$. Therefore,
$$\frac{\mbox{vol}\left(X_{i_{\hat{k}}}\right)}{\mbox{vol}\left(X_{i_{\hat{k}-1}}\right)} \geq \frac{13}{10}$$

The algorithm will terminate when $X_{i_{\tilde{k}}} \supseteq E(X_{i_{\tilde{k}}}) \supseteq B_k(0,\frac{1}{2\sqrt{k}})$. So for any unit vector $\bm{a}\in B_k(0,1)$, $d(\bm{a}, \bm{V}_{i_{\tilde{k}}}) \leq \frac{\epsilon}{c''}$. This means that after learning $\tilde{k}$ new features, for any new tasks with weights lie in the same feature subspace $\bm{V}^*$, the current features can achieve error less than $\epsilon$.
By Lemma \ref{lem:maxvol},we know the volume of $E(X_{i_{\tilde{k}}})$ is upper bounded by $2\sqrt{2e}\left(\sqrt{\frac{2\log 2\tilde{k}}{k}}\right)^k\vol(B_k(0,1))$. It grows by a constant factor $\frac{13}{10}$ whenever we learn a new feature. So the number of tasks we learn with error $\eps_{acc}$ in the algorithm $\tilde{k}$ satisfies:
\begin{align*}
    \left(\frac{\epsilon}{2c''\sqrt{k}}\right)^k\cdot \left(\frac{13}{10}\right)^{\tilde{k}} \leq 2\sqrt{2e}\left(\sqrt{\frac{2\log 2\tilde{k}}{k}}\right)^k
\end{align*}
Simplify and take the log to both sides, so we will have $
    \tilde{k}-\frac{k}{2}\log(2\log(2\tilde{k})) \leq const+k\log(\frac{1}{\eps})
$. This will lead to $\tilde{k} \leq O(k\log(\log(k)/\eps))$.

Moreover, the sample complexity \cite{balcan2013active} of learning one task with input dimension $d$ up to $\epsilon$ error is $O(d\log(1/\eps)/\epsilon)$. So the sample complexity of our algorithm is 
\begin{align*}
&O\left(\frac{dk}{\eps_{acc}}\log\left(1/\eps_{acc}\right)\log\left(\log\left(k\right)/\eps\right)\right) +O\left(\frac{km}{\eps}\log\left(1/\eps\right)\right) 
\\
&= O\left(\frac{dk\sqrt{k}}{\eps}\log\left(k/\eps\right)\log\left(\log\left(k\right)/\eps\right)+ \frac{km}{\eps}\log\left(1/\eps\right)\right)\\ &=\tilde{O}\left(\frac{dk\sqrt{k}}{\eps}+\frac{km}{\eps}\right).
\end{align*}

Finally, for the nonlinear case, we consider the kernel of the features. These features live in a potentially infinite-dimensional space. If we assume there is an oracle to get a constant approximation for the optimization problem~\ref{equation:original_opt}, the dimension of features will be $O(k)$ in the end. Other bounds follow precisely the same as the linear case.
\end{proof}

It is noteworthy that the convex set $X_i$ (of feature vectors in $ \mathbb{R}^k$) we keep in the proof is defined to be close to any possible subspaces that are close to the subspace spanned by new features $\tilde{\bm{w}}_i$. So our proof is quite general for any lifelong learning algorithm that dynamically expands the architecture, \eg the basic LLL algorithm. Consequently, we can prove Theorem~\ref{thm:vanilla_algo} as follows.

\begin{proof}[Proof of Theorem~\ref{thm:vanilla_algo}.]
The proof exactly follows the proof of the Theorem \ref{thm:tune_linear}. Without the refinement of the feature subspace, the subspace we get is still in the set $Y_{i_{\tilde{k}}}$. Since $X_{i_{\tilde{k}}}$ will eventually cover the $B_k(0,\frac{1}{2\sqrt{k}})$ after learning $O(k\log(\log(k)/\epsilon))$ features, the dimension of the feature subspace is at most $O(k\log(\log(k)/\epsilon))$ for the linear features and $O(k^2\log(\log(k)/\epsilon))$ for the nonlinear features. So the total sample complexity is $\tilde{O}(dk^{1.5}/\eps+km/\eps)$ for linear features and $\tilde{O}(dk^{2.5}/\eps+k^2m/\eps)$ for nonlinear features.
\end{proof}

\subsection{Proof of the Approximate Optimization for the Linear Setting}\label{section:refine_approx}

\begin{proof}[Proof of Theorem~\ref{thm:sdp_approx}.]
Let $(\bm{X}^*, t^*)$ be a solution of the SDP~(\ref{equation:sdp_relation}) with singular value decomposition (SVD) $\bm{X}^*=\sum_{i=1}^d\lambda_i \bm{u}_i \bm{u}_i^\top$, where $0\leq\lambda_1\leq \cdots\leq \lambda_{d}\leq 1$, and $\sum_{i=1}^d \lambda_i=d-k$. Let $\bm{V}'$ be the span of $\{\bm{u}_1,\cdots, \bm{u}_{k'}\}$. Then the squared distance of any vector $\bm{a}$ to $\bm{V}'$ is $\sum_{i=k'+1}^d (\bm{a}^\top\bm{u}_i)^2$. In the meantime, the SDP assigns a value of $\bm{a}^\top \bm{X}^*\bm{a} = \sum_{i=1}^d \lambda_i (\bm{a}^\top \bm{u}_i)^2$. Thus, the multiplicative increase in squared distance is at most
\[
\frac{\sum_{i=k'+1}^d (\bm{a}^\top \bm{u}_i)^2}{\sum_{i=1}^d \lambda_i (\bm{a}^\top \bm{u}_i)^2} \leq \frac{1}{\lambda_{k'+1}}
\]
Now since the sum of all eigenvalues is $d-k$ and each one is at most $1$, for $k'\geq k$, we must have
\[
\lambda_{k'+1} \geq \frac{(d-k) - (d-k'-1)}{k'+1}=\frac{k'-k+1}{k'+1}
\]
Choose $k'=ck-1$, and we would have
\[
\frac{\sum_{i=ck}^d (\bm{a}^\top \bm{u}_i)^2}{\sum_{i=1}^d \lambda_i (\bm{a}^\top \bm{u}_i)^2} \leq \frac{1}{\lambda_{ck}} \leq 1+\frac{1}{c-1}
\]
Since there exists a subspace $\bm{V}^*$ of $k$ dimension with $d(\tilde{\bm{w}}_i, \bm{V}^*) \leq \epsilon_{acc}, \forall i\in[\hat{k}+1]$, we have $\sum_{i=1}^d \lambda_i (\tilde{\bm{w}}_i^\top \bm{u}_i)^2 \leq \epsilon_{acc}^2$. Consequently, $\lspan(\bm{u}_1,\cdots,\bm{u}_{ck-1})$ is a $(ck-1)$-dimensional approximation of $\bm{V}^*$ with the approximation factor $\sqrt{1+\frac{1}{c-1}}$ in maximum distance. 

Specifically, for $c=2$, we have $d^2(\tilde{\bm{w}}_i,\bm{V}') =\sum_{i=2k}^d (\tilde{\bm{w}}_i^\top \bm{u}_i)^2 \leq 2\epsilon_{acc}^2$.
\end{proof}

Theorem~\ref{thm:sdp_approx} guarantees that in the linear case, our refinement step will always output an $O(k)$-dimensional feature subspace of error bounded by $O(\epsilon_{acc})$ without using any additional samples.

\section{A Lower Bound for General Lifelong Learning Algorithms}\label{section:adv_lower_bound}

In this section, we show that our sample complexity bound for \textit{general} lifelong learning algorithms in the linear setting is asymptotically the best possible, assuming black-box access to a learner for a single linear target. 

As a warm-up, we first show that the analysis of \textit{our} lifelong algorithm is \textit{tight}.

\begin{thm}[Tight Example]\label{thm:lower_bound}
Using the same condition and algorithm as in Theorem \ref{thm:tune_linear}, the total sample complexity is $\Omega(dk^{1.5}/\epsilon+km/\epsilon)$.
\end{thm}

\begin{proof}
Assume the task vectors $\bm{a}_i=\bm{e}_i,1\leq i\leq k-1$, where $\bm{e}_i\in\mathbb{R}^k$ has 1 in $i$-th coordinate and 0 otherwise. Assume that our algorithm accurately learns these $k-1$ tasks and returns $\tilde{\bm{w}}_{i}=\bm{a}_i+\epsilon_{acc} \bm{a}_k$. Then for a new task's weight $\frac{1}{\sqrt{k}}\sum_{i=1}^{k-1} \alpha_i \bm{a}_i$, $\alpha_i\in\{1,-1\}$. Based on the current features, the error we make on this task is $O(\sqrt{k}\epsilon_{acc})$. If we assume that $\epsilon_{acc}=\omega(\epsilon/\sqrt{k})$, then we need to learn these $2^{k-2}$ tasks accurately as well given learning any of them will not help with others (except its negative). Then the total complexity will be exponential with respect to $k$. So $\epsilon_{acc}=O(\epsilon/\sqrt{k})$, and thus we have the sample complexity $\Omega(dk^{1.
5}/\eps+km/\eps)$.
\end{proof}

Theorem~\ref{thm:lower_bound} shows that the analysis of our algorithm's sample complexity is tight for the linear setting. The main result of this section is a lower bound for general lifelong learning algorithms.

\begin{manualtheorem}{2}[Lower bound]
Suppose that a lifelong learner has black-box access to a single task learner that takes an error parameter $\eps$ as input and is allowed to return any vector that is within distance $\eps$ of the true target unit vector, using $\Theta(d/\eps)$ samples in $\R^d$. Then, there exists a distribution of $m$ tasks, $m = 2^{\Theta(k)}$ such that for any lifelong learning algorithm, WHP, the total number of samples required to learn all $m$ tasks up to error $\eps$ is $\Omega(dk^{1.5}/\eps+km/\eps)$.
\end{manualtheorem}

\paragraph{Proof idea.}

\begin{figure}[htbp]
    \centering
    \includegraphics[height=1.8in]{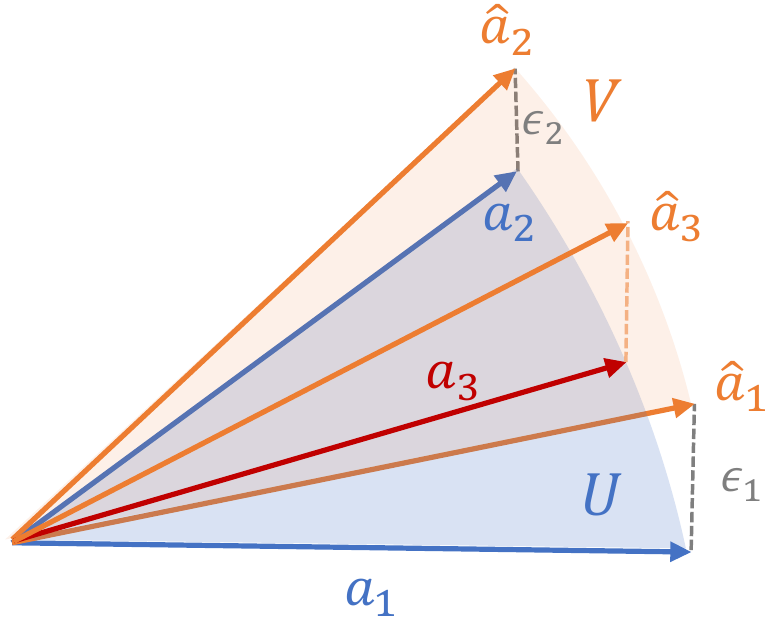}
    \caption{Geometric illustration of the lower bound examples when $k=2,d=3$. The errors that the algorithm makes concentrate on the third coordinate, and thus lead to the large angle between the learned feature subspace $V$ and the underlying one $U$. Any new task, \eg, $\bm{a}_3$ that lies on the span of $\{\bm{a}_1,\bm{a}_2\}$ cannot help improve the representation $V$.}
    \label{fig:lower_bound_intuition}
\end{figure}

Consider a sequence of tasks $\bm{a}_i=
\begin{cases}
\bm{e}_i, & 1\leq i\leq k\\
\sum_{j\in S} x_{j}\bm{e}_j, &  i> k
\end{cases}$, where $\bm{e}_j$ is the standard unit vector, $x_{j}\stackrel{\text{i.i.d}}{\sim}\text{Bernoulli}(1/2)$, $S\subset [k]$ is a subset of indices. 
The proof is mainly by constructing an adversarial output of the algorithm where the errors that it makes concentrate on one coordinate, say $k+1$. (See Figure~\ref{fig:lower_bound_intuition}.)
We will show by the following steps:
\begin{enumerate}
    \item[(1)] After learning the first $k$ tasks, each with error $\eps_i$, the angle between the learned feature subspace and the underlying one is at least $\Omega(\sqrt{\sum_{i=1}^k \eps_i^2})$.
    \item[(2)] For each new task followed, the angle of the new task to the learned subspace is at least $\Omega(\sqrt{\sum_{i\in S} \eps_i^2})$ with high probability.
    \item[(3)] Learning such a new task does not improve the representation.
    \item[(4)] To solve all tasks up to error $\eps$, we will need each $\eps_i = O(\eps/\sqrt{k})$, and it leads to the sample complexity. 
\end{enumerate}
We will show them accordingly with the following lemmas.

\begin{lemma}\label{lemma:k_task_angle}
For $k$ orthonormal tasks $\bm{a}_i=\bm{e}_i,1\leq i\leq k$, for any algorithm that learns task $i$ within error $\eps_i$, \ie, $d_D(\bm{a}_i, \hat{\bm{a}}_i) \leq \eps_i$. Let $U=\lspan(\bm{a}_1,\cdots, \bm{a}_k)$ be original feature subspace, and $V=\lspan(\hat{\bm{a}}_1,\cdots, \hat{\bm{a}}_k)$ be the learned subspace. Then there exist feasible outputs $\hat{\bm{a}}_i$ such that $\theta(U,V)= \Omega(\sqrt{\sum_{i=1}^k \eps_i^2})$.
\end{lemma}

\begin{proof}
Since the algorithm learns task $i$ within error $\eps_i$, with Lemma~\ref{lemma:error2dist}, there exists a constant $c$ such that $\|\bm{a}_i-\hat{\bm{a}}_i\| \leq \eps_i/c$. Because we does not consider constant factor, we assume WLOG that $\|\bm{a}_i-\hat{\bm{a}}_i\| \leq \eps_i$.
Consider the case where all errors made by the algorithm concentrate on the $k+1$ coordinate. Then the features learned by the algorithm are 
$\hat{\bm{a}_i} = 
\begin{pmatrix}
 \bm{e}_i\\
 \eps_i
\end{pmatrix}, 1\leq i\leq k$.
Denote $\bm{A}=(\bm{a}_1,\cdots, \bm{a}_k)=
\begin{pmatrix}
 \bm{I}_k\\
 \bm{0}^\top
\end{pmatrix}$, columns of which are the task vectors. Denote $\bm{s}=(\eps_1,\cdots,\eps_k)^\top$. $\hat{\bm{A}}=(\hat{\bm{a}}_1,\cdots, \hat{\bm{a}}_k) = \begin{pmatrix}
 \bm{I}_k\\
 \bm{s}^\top
\end{pmatrix}$. 
By the definition of the angles between subspace, we have $\theta(U, V)= \max\limits_{\bm{x\in U}}\theta(\bm{x}, \bm{P}_V\bm{x})$, where $\bm{P}_V$ is the projection matrix of the subspace $V$.
\begin{align*}
    \bm{P}_V =& \hat{\bm{A}} (\hat{\bm{A}}^\top \hat{\bm{A}})^{-1}\hat{\bm{A}}^\top\\
    =&\begin{pmatrix}
     \bm{I}_k\\
     \bm{s}^\top 
    \end{pmatrix}
    \begin{pmatrix}
     \bm{I}_k+\bm{s}\bm{s}^\top
    \end{pmatrix}^{-1}
    \begin{pmatrix}
     \bm{I}_k & \bm{s}
    \end{pmatrix}\\
    =&\begin{pmatrix}
     \bm{I}_k\\
     \bm{s}^\top
    \end{pmatrix}
    \begin{pmatrix}
     \bm{I}_k-\bm{s}\bm{s}^\top + o((\bm{s}\bm{s})^3)
    \end{pmatrix}
    \begin{pmatrix}
     \bm{I}_k &  \bm{s}
    \end{pmatrix}\\
    \approx&\begin{pmatrix}
     \bm{I}_k-\bm{s}\bm{s}^\top  &  \bm{s}(1-\bm{s}^\top\bm{s}) \\
      \bm{s}^\top(1-\bm{s}^\top\bm{s})  &  \bm{s}^\top\bm{s} - (\bm{s}^\top\bm{s})^2
    \end{pmatrix}
\end{align*}
Let $\bm{x}=\bm{s}=(\eps_1,\cdots,\eps_k)^\top\in U$.
Since 
\[
\bm{P}_V \bm{x}= \left(\left(1-\sum_{i=1}^k \eps_i^2\right)\eps_1,\cdots, \left(1-\sum_{i=1}^k \eps_i^2\right)\eps_k, \left(1-\sum_{i=1}^k \eps_i^2\right)\sum_{i=1}^k \eps_i^2\right)^\top, 
\]
we have 
\begin{align*}
    \tan \theta\left(\bm{x},\bm{P}_V\bm{x}\right) = 
    \frac{\left(1-\sum_{i=1}^k \eps_i^2\right)\sum_{i=1}^k \eps_i^2}{\left(1-\sum_{i=1}^k \eps_i^2\right)\sqrt{\sum_{i=1}^k \eps_i^2}}
    = \sqrt{\sum_{i=1}^k \eps_i^2}
\end{align*}
So we know that $\theta(U,V)\geq \theta(\bm{x},\bm{P}_V\bm{x})=\Omega(\sqrt{\sum_{i=1}^k \eps_i^2})$. So we prove that the angle between the learned subspace $V$ and the underlying one $U$ is at lease $\Omega(\sqrt{\sum_{i=1}^k \eps_i^2})$.

\end{proof} 

\begin{lemma}\label{lemma:new_task_angle}
For any vector $\bm{x}=\sum_{i\in S} x_{i}\bm{e}_i$, where $S\subset [k]$, 
$x_{i}\stackrel{\text{i.i.d}}{\sim}\text{Bernoulli}(1/2)$. Assume that $\forall i\in S, 0\leq \eps_i \leq 2\sqrt{\sum_{i\in S}\eps_i^2/|S|}$. Let $V=\lspan(\{\hat{\bm{a}}_i\}, i\in S)$, where $\hat{\bm{a}_i} = 
\begin{pmatrix}
 \bm{e}_i\\
 \eps_i
\end{pmatrix}$. Then with high probability, $\theta(\bm{x}, V)=\Omega (\sqrt{\sum_{i\in S} \eps_i^2})$.
\end{lemma}

\begin{proof}
Denote $\bm{y}=\bm{P}_V(\bm{x})$, then $y_j=\begin{cases}
x_j-\eps_j \sum_{i\in S} \eps_i x_i & \text{ if }j\in S\\
(1-\sum_{i=1}^k\eps_i^2)\sum_{i\in S} \eps_i x_i  & \text{ if }j=k+1
\end{cases}$. 
Then we can calculate the angle between $\bm{x}$ and the subspace $V$ as follows.

\begin{align*}
    \left(\tan\left(\theta\left(\bm{x}, \bm{P}_V\left(\bm{x}\right)\right)\right)\right)^2 
    =& 
    \frac{\left(1-\sum_{i\in S} \eps_i^2\right)^2 \left(\sum_{i\in S} \eps_ix_i\right)^2 }
    {\sum_{i\in S}x_i^2  + \left(\sum_{i\in S} \eps_i^2\right)\left(\sum_{i\in S} \eps_i x_i\right)^2 - 2\left(\sum_{i\in S} \eps_ix_i\right)^2}
    \geq
    \frac{ \frac{1}{4} \left(\sum_{i\in S} \eps_ix_i\right)^2}
    {\sum_{i\in S} x_i^2}
\end{align*}
The last inequality is because we learn each task well, so we can assume $\sum_{i\in S} \eps_i^2 \leq 1/2$. Then the probability that $\theta(\bm{x}, \bm{P}_V(\bm{x}))$ is greater than $O(\sqrt{\sum_{i\in S} \eps_i^2})$ is as follows.
\begin{align*}
    \mathbb{P}\left(\theta\left(\bm{x},\bm{P}_V\left(\bm{x}\right)\right)\geq \frac{1}{16}\sqrt{\sum_{i\in S}{\eps_i^2}}\right)
    =&
    \mathbb{P}\left(\tan^2\theta(\bm{x},\bm{P}_V\left(\bm{x}\right)\right)\geq \frac{1}{256}\sum_{i\in S} \eps_i^2)\\
    \geq& \mathbb{P}\left(\frac{\frac{1}{4}\left(\sum_{i\in S} \eps_ix_i\right)^2}{\sum_{i\in S} x_i^2} \geq \frac{1}{256}\sum_{i\in S} \eps_i^2\right)\\
    =& \mathbb{P}\left(\left(\sum_{i\in S} \eps_ix_i\right)^2\geq \frac{1}{64}\sum_{i\in S} \eps_i^2 \sum_{i\in S} x_i^2\right)\\
    \geq &  \mathbb{P}\left(\left(\sum_{i\in S} \eps_ix_i\right)^2\geq \frac{1}{64}|S|\sum_{i\in S} \eps_i^2\right)
\end{align*}
By Chernoff bound, for any $t>0$,
\[
\mathbb{P}\left(\sum_{i\in S} \eps_i x_i - \frac{1}{2}\sum_{i\in S}\eps_i \geq - t\sqrt{\sum_{i\in S} \eps_i^2}\right) \geq 1- e^{-t^2/2}
\]
Choose $t=\sqrt{|S|}/8$, we have
\[
\mathbb{P}\left(\sum_{i\in S} \eps_i x_i \geq \frac{1}{2}\sum_{i\in S} \eps_i -\frac{1}{8}\sqrt{|S|\sum_{i\in S}\eps_i^2}\right) \geq 1-e^{-|S|/128}
\]
Note that with $\eps = \sqrt{\sum_{i\in S} \eps_i^2}$ and 
$0\leq \eps_i \leq 2\eps/\sqrt{|S|}$ for all $i\in S$, we have
\[
\sum_{i\in S} \eps_i \ge \frac{\sqrt{|S|}}{2}\eps.
\]
Consequently, we have 
\[
\sqrt{\frac{1}{64}|S|\sum_{i\in S}\eps_i^2} \leq \frac{1}{2}\sum_{i\in S}\eps_i - \frac{1}{8}\sqrt{|S|\sum_{i\in S} \eps_i^2}
\]
\begin{align*}
    \mathbb{P}\left(\theta\left(\bm{x},V\right)\geq \frac{1}{16}\sqrt{\sum_{i\in S}{\eps_i^2}}\right)
    \geq &
     \mathbb{P}\left(\left(\sum_{i\in S} \eps_ib_i\right)^2\geq \frac{1}{64}|S|\sum_{i\in S} \eps_i^2\right)\\
     \geq &
     \mathbb{P}\left(\sum_{i\in S} \eps_i b_i \geq \frac{1}{2}\sum_{i\in S} \eps_i -\frac{1}{8}\sqrt{|S|\sum_{i\in S}\eps_i^2}\right) \\
     \geq &
     1-e^{-|S|/128}
\end{align*}
\end{proof}

\begin{lemma}\label{lemma:find_avg_subset}
For $b_1,\cdots, b_k \geq 0$, with $\bar{b} = \sqrt{\sum_{i=1}^k b_i^2/k}$, let $b_i \ge \bar{b}/C$ for some constant $C > 1$. Then 
there exists a subset $S\subset [k]$ with $|S| \geq k(1-p)$ s.t. for all $i \in S$, we have 
\[
b_i \leq\sqrt{\frac{1}{p}\ln \frac{C^2}{1-p}}  \sqrt{\frac{\sum_{i\in S} b_i^2}{|S|}}.
\]
\end{lemma}

\begin{proof}
Let $S_1=\{1,\cdots,k\}$. Choose a constant $\gamma>1$. We repeat the following procedure. For the $j$-th step, we are given a set $\{x_i,i\in S_j\}$, Let $S_{j+1} = \{i\in S, b_i \leq \gamma \sqrt{\sum_{i\in S_j}b_i^2/|S_j|}\}$. The algorithm terminates when $S_j=S_{j+1}$.
Denote $p_j=1 - |S_{j+1}|/|S_j|$. For the $j$-th step, we have
\[
\sum_{i\in S_j} b_i^2
= \sum_{i \in S_{j+1}} b_i^2 + \sum_{i \in S_j\backslash S_{j+1}} b_i ^2
\geq \sum_{i \in S_{j+1}} b_i^2 +  p_j \gamma^2 \sum_{i\in S_j}b_i^2
\] 
This derives that
\[
\sum_{i\in S_{j+1}}b_i^2 \leq (1-p_j \gamma^2)\sum_{i\in S_j}b_i^2 \leq e^{-p_j\gamma^2}\sum_{i\in S_j}b_i^2
\]
Accumulating all $J$ steps, we have
\[
\sum_{i\in S_J}b_i^2 \leq e^{-\gamma^2\sum_{j=1}^{J-1}p_j}\sum_{i\in S_1}b_i^2 =  e^{-\gamma^2\sum_{j=1}^{J-1}p_j} k\bar{b}^2
\]
From the condition that $b_i \geq \bar{b}/C$, we have
\[
|S_J|\frac{\bar{b}^2}{C^2}
\leq 
\sum_{i\in S_J}b_i^2 \leq e^{-\gamma^2\sum_{j=1}^{J-1}p_j} k\bar{b}^2
\]
From the definition of $p_j$, we know that 
\[
|S_J|=k\prod_{j=1}^{J-1} \left(1-p_j\right) \geq k\left(1-\sum_{j=1}^{J-1}p_j\right)
\]
Denote $p=\sum_{j=1}^{J-1}p_j$, then we have
\[
1-p \leq e^{-\gamma^2 p} C^2
\]
Choose $\gamma=\sqrt{\frac{1}{p}\ln \frac{C^2}{1-p}} $, then we get a set $S_J$ with $|S_j|\geq k(1-p)$ satisfying for all $i\in S_J$,
\[
b_i \leq \sqrt{\frac{1}{p}\ln \frac{C^2}{1-p}} \sqrt{\frac{\sum_{i\in S_J}b_i^2}{|S_j|}}
\]

\end{proof}

We can now prove Theorem~\ref{thm:adv_lower_bound} as follows.
\begin{proof}
Denote the underlying feature subspace as $U$.
Consider a sequence of tasks with first $k$ tasks the basis of the feature subspace, \ie, $\bm{a}_i=\bm{e}_i,  1\leq i\leq k$. The lifelong learning algorithm learns task $i$ up to error $\eps_i$, and get $k$ features $\hat{\bm{a}_1},\cdots,\hat{\bm{a}}_k$. 
If the number of tasks for which $\eps_i < \sqrt{2\sum_{i=1}^k \eps_i^2/3k}$ is more than $k/4$, then we have that the total sample complexity is already $\Omega(dk^{1.5}/\sqrt{\sum_{i=1}^k \eps_i^2})$ and the theorem follows.
So we assume that for at least $3k/4$ tasks, we have 
\[
\eps_i \ge \sqrt{2\sum_{i=1}^k \eps_i^2/3k}. 
\]
Calling this subset $S_1$, it follows that for each $i \in S_1$, we have 
$\eps_i \ge \sqrt{\sum_{i\in S_1} \eps_i^2/2|S_1|}$. 
Next, applying Lemma~\ref{lemma:find_avg_subset} to the set $S_1$  with cardinality at least $3k/4$, using $p=1/3$ and $C = \sqrt{2}$, we get that there exists a set $S\subseteq S_1$ with $|S| \geq k/2$ and $\eps_i \leq 2\sqrt{\frac{\sum_{i\in S}\eps_i^2}{|S|}}$ for all $i\in S$. Consider the span $V:=\{\hat{\bm{a}}_i, i\in S\}$. By Lemma~\ref{lemma:k_task_angle}, we know there exists feasible $\hat{\bm{a}}_i$ such that $\theta(U,V) =\Omega(\sqrt{\sum_{i\in S}\eps_i^2})$.

Next we consider the following tasks as $\bm{a}_j=\sum_{i\in S}x_{ji}\bm{e}_i, j\geq k+1$, where $x_{ji}\stackrel{\text{i.i.d}}{\sim}\text{Bernoulli}(1/2)$. There are $2^{k/2}$ such tasks. By Lemma~\ref{lemma:new_task_angle}, we know that with high probability each new task is far from the learned subspace, \ie,
\[
\mathbb{P}\left(\theta\left(\bm{a}_j,V\right) \geq \frac{1}{16}\sqrt{\sum_{i\in S}\eps_i^2}\right) \geq 1-e^{-k/256}
\]
Assume that the single-task learning algorithm applied to this new task also induce error in the $(k+1)$'st coordinate. Say the new learned feature $\hat{\bm{a}}_j = (\bm{a}_j^\top, \sum_{i\in S} b_{ji}\eps_i)^\top$. Then the new learned feature is in the learned subspace $V$, which means that learning new tasks does not improve the learned subspace. 

Since each new task is generated randomly, there are exponentially many new tasks that are far from the learned subspace but make no improvement by learning them. Therefore to ensure each task is learned with error $\eps$, the only way is to let $\theta(U,V) \leq \eps$. This implies that $\sum_{i\in S}\eps_i^2 \leq c\eps^2$. By the generalized H\"{o}lder inequality~\cite{finner1992generalization}, we have
$$\sum_{i\in S}\frac{1}{\eps_i}\sum_{i\in S}\frac{1}{\eps_i}\sum_{i\in S}\eps_i^2 \geq
\left(\sum_{i\in S} 1\right)^3
\geq\left(\frac{k}{2}\right)^3$$
So $\sum_{i\in k}\frac{1}{\eps_i} \geq c'k^{1.5}/\eps$. The number of the samples needed to learn the tasks in set $S$ is $\sum_{i\in S}\frac{d}{\eps_i} = \Omega(dk^{1.5}/\eps)$. So the overall sample complexity for the sequence of tasks are $\Omega(dk^{1.5}/\eps+km/\eps)$.

\end{proof}

\section{Simulations and Empirical Results}\label{section:experiments}

In this section, we describe our experimental studies. In Section~\ref{section:exp_linear}, we run the basic LLL and LLL-RR algorithms in a task-incremental binary classification setting. Then we conduct class-incremental experiments on real dataset using our H-LLL algorithm in Section~\ref{section:real_data_exp}. The performance shows the benefits of our algorithm compared to existing continual learning algorithms.

\subsection{Linear Features}\label{section:exp_linear}

\begin{figure}[ht]
\centering
\begin{minipage}[t]{.41\linewidth}
\centering
\includegraphics[height=2.1in]{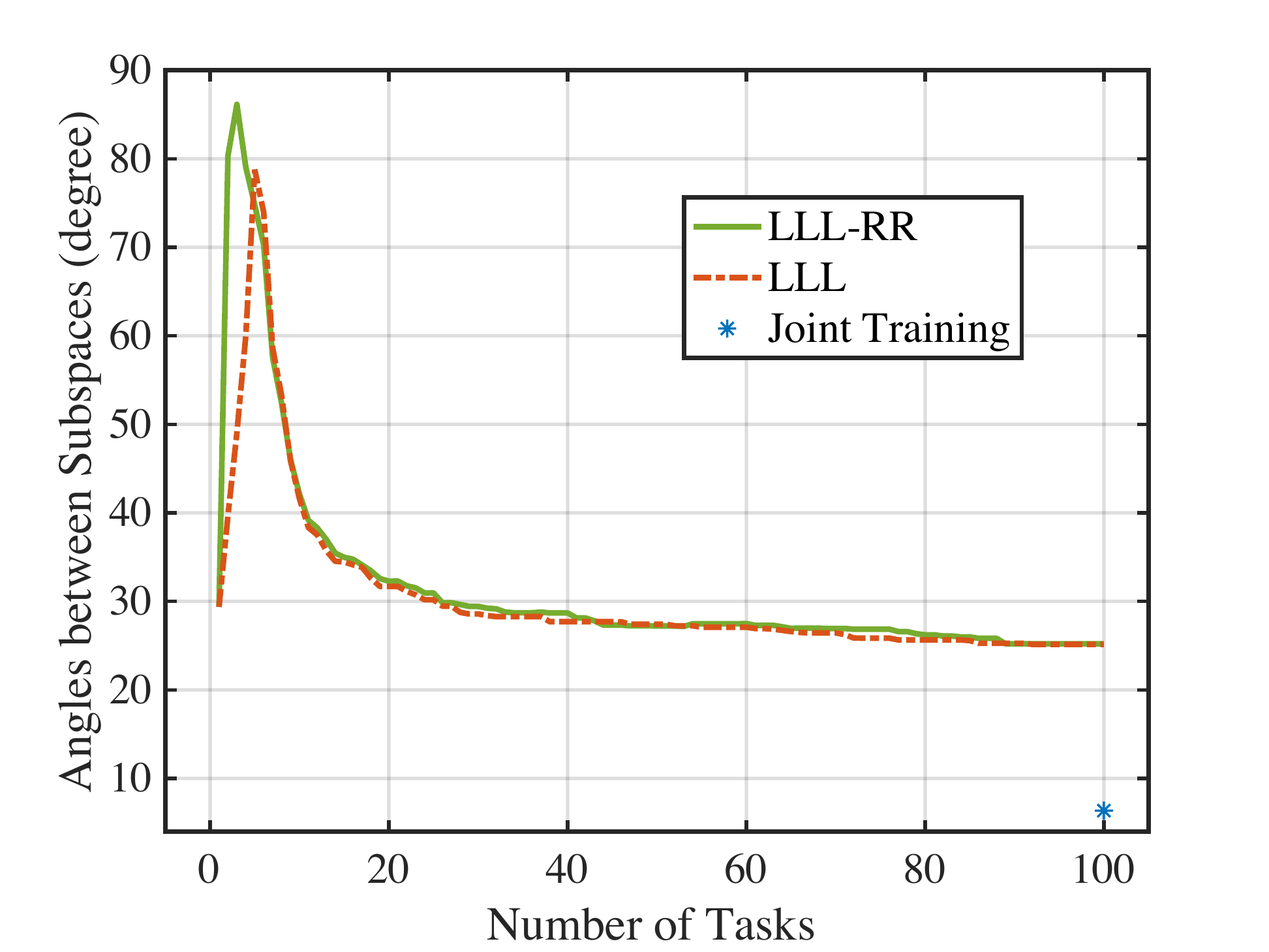}
\end{minipage}
\begin{minipage}[t]{.41\linewidth}
\centering
\includegraphics[height=2.1in]{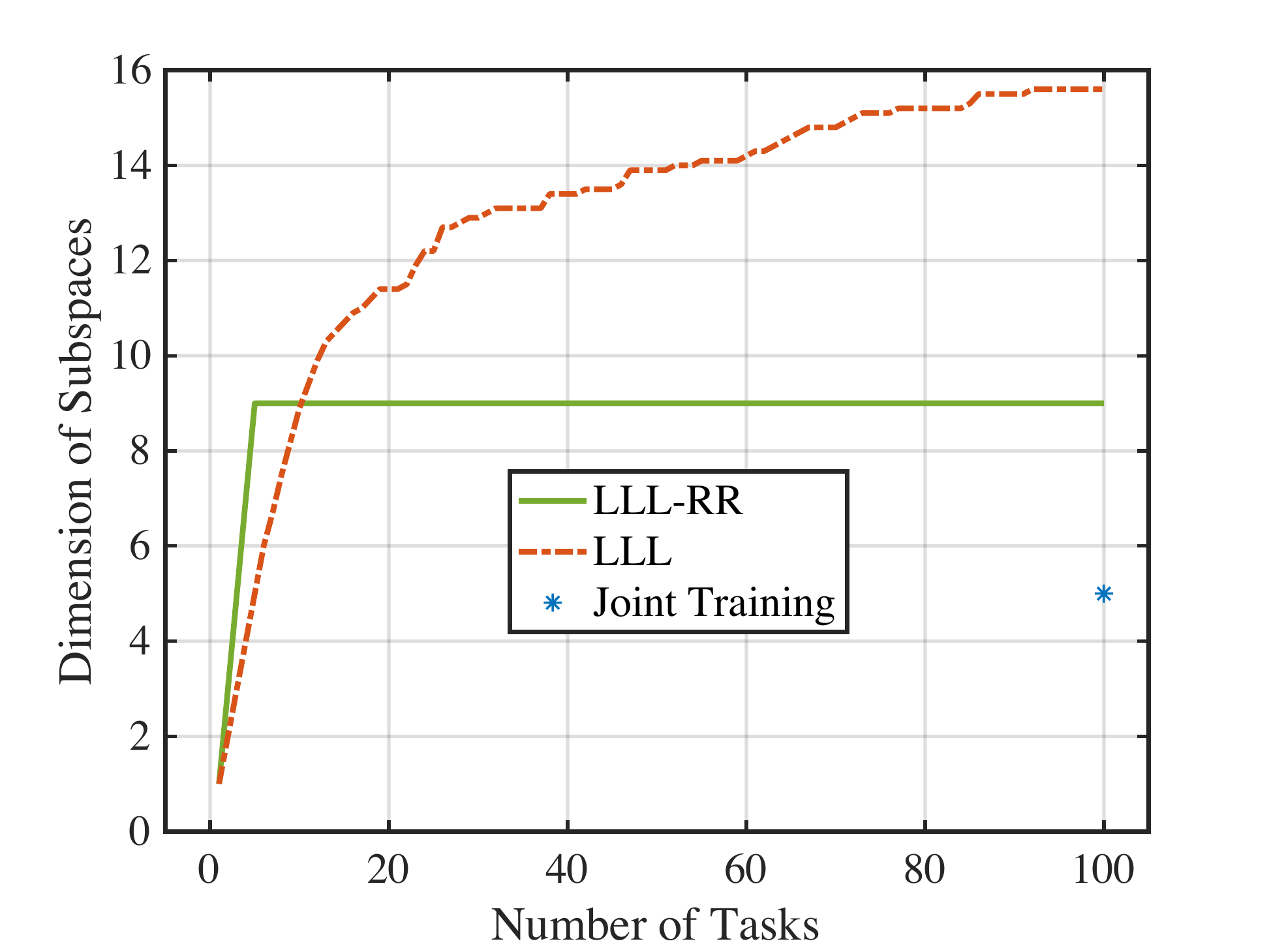}
\end{minipage}
\begin{minipage}[t]{.41\linewidth}
\centering
\includegraphics[height=2.1in]{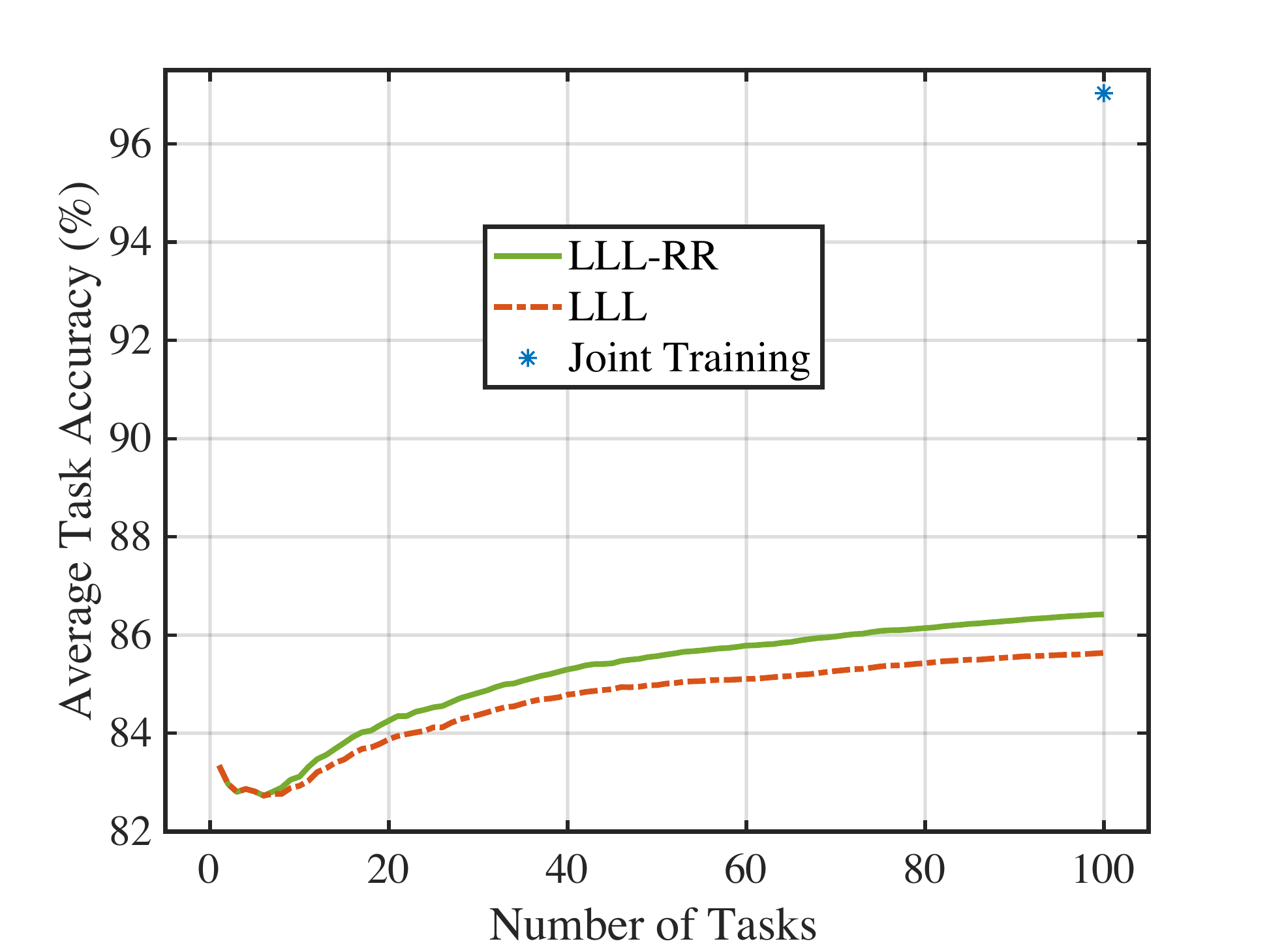}
\end{minipage}
\begin{minipage}[t]{.41\linewidth}
\centering
\includegraphics[height=2.1in]{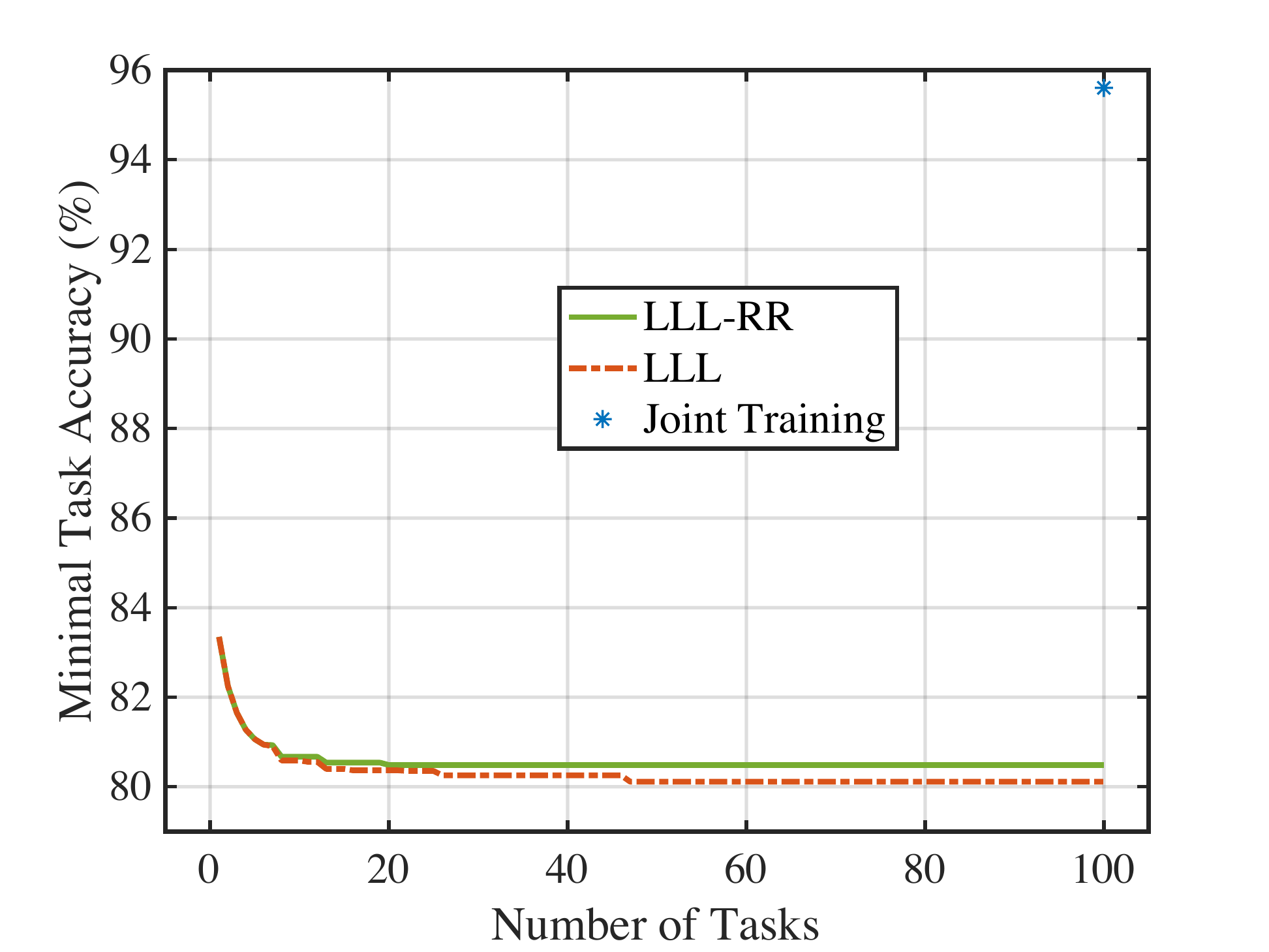}
\end{minipage}
\caption{Simulation on linear features. $\thickmuskip=2mu \medmuskip=2mu N=200,k=5,m=100,d=100$, averaged on 10 trials.}\label{fig:linear_lll}
\end{figure}
Here we consider task-incremental lifelong learning in the setting of binary classification where $y=\sign(\langle \bm{c}_i^*, \bm{W}^*\bm{x}\rangle)$. We choose the input dimension, $d=100$, the number of tasks, $m=100$, the number of examples per task, $N=200$, the dimension of feature subspace, $k=5$. The parameters $c^*_{ij}, W^*_{ij} \sim \mathcal{N}(0,1)$. The input data $X_i\sim \mathcal{N}(0, 1)$. We set the error threshold to be $\epsilon=0.1$. We compare three methods: LLL (Basic lifelong learning algorithm), LLL-RR (lifelong learning algorithm with representation refinement), and Joint Training (offline training with all data jointly).

The average task accuracy and minimal task accuracy are computed for tasks encountered so far based on the current model. The angle between feature subspaces is calculated as their maximal principal angle. Formally, for two subspaces $F$ and $G$, let $\bm{P}, \bm{Q}$ to be the orthogonal matrices whose columns form an orthonormal basis of $F$ and $G$. For the singular value decomposition $\bm{P}^\top \bm{Q}=\bm{U\Sigma V}^\top$, we define the principal angles between $F$ and $G$ as $\theta_i=\arccos(\bm{\Sigma}_{ii})$, $\frac{\pi}{2} \geq \theta_1\geq \cdots,\geq \theta_k \geq 0$. We calculate the angle between two subspaces $F$ and $G$ as the maximal principal angle, \ie, $\arccos(\bm{\Sigma}_{11})$.

As we can see in Figure \ref{fig:linear_lll}, lifelong learning can continually learn better features while learning more tasks. Moreover, lifelong learning with refinement improves average accuracy, min accuracy, model size and convergence to the underlying feature subspace.

\subsection{Image Classification}\label{section:real_data_exp}

\paragraph{Experimental settings.} We generally follow the experimental settings and evaluation protocol in \cite{rebuffi2017icarl}. In our experiments, we evaluate our H-LLL algorithm on CIFAR-100. We train all 100 classes in 10 splits and each split contains 10 classes. There is no class overlap between different splits. Each training data split can be viewed as a task and is fed to the neural network incrementally. Similar to \cite{rebuffi2017icarl}, we use a fixed memory size of 2,000 exemplars. The final result are curves of the classification accuracies after each batch of classes. We use ResNet-18~\cite{he2016deep} for all the encoders $f_j,\forall j$ and SGD with weight decay $0.0005$. All the ResNet encoders are trained from scratch.

\begin{figure}[t]
\centering
\begin{minipage}[t]{.41\linewidth}
\centering
\includegraphics[height=2.1in]{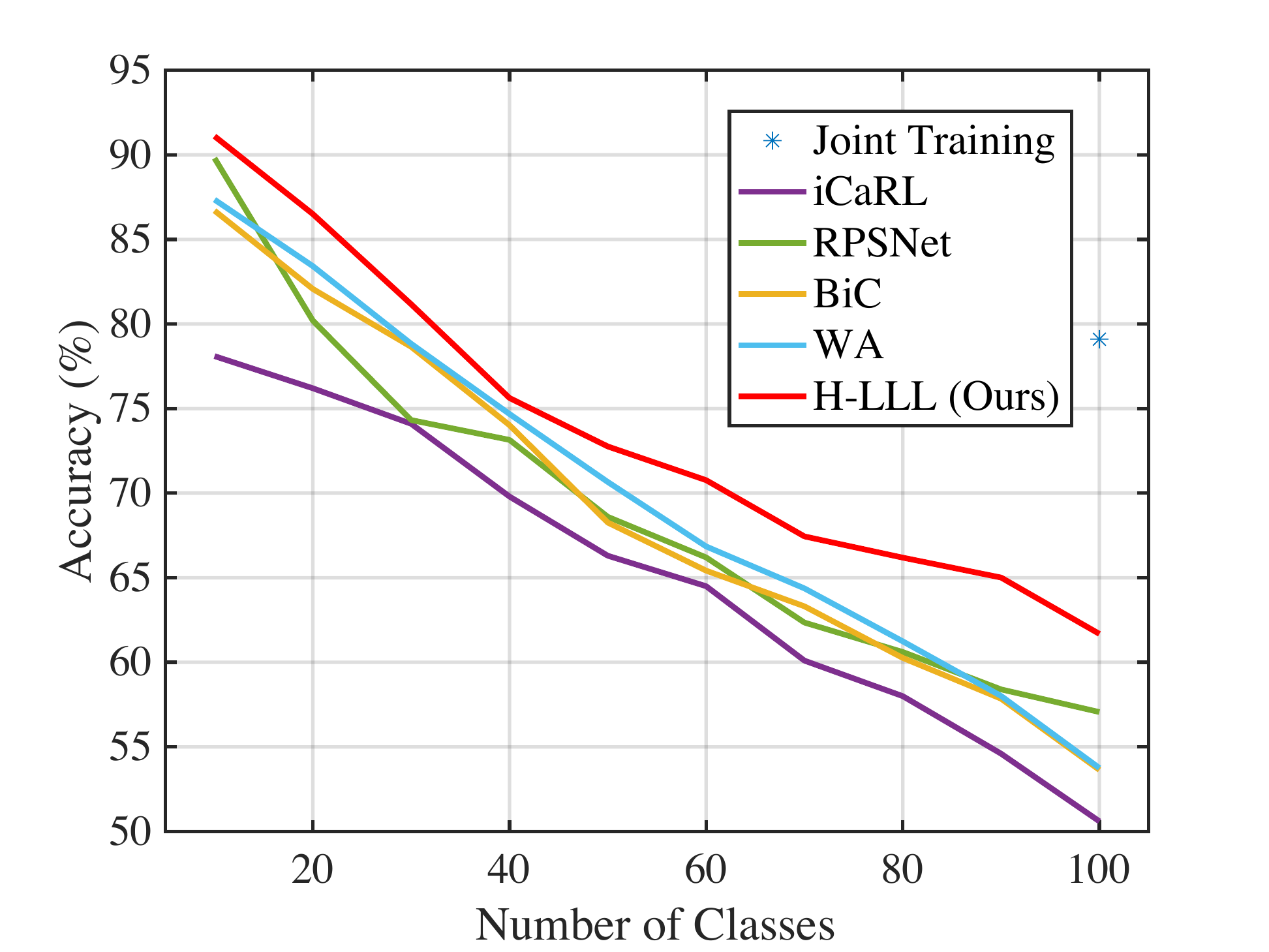}
\end{minipage}
\begin{minipage}[t]{.41\linewidth}
\centering
\includegraphics[height=2.1in]{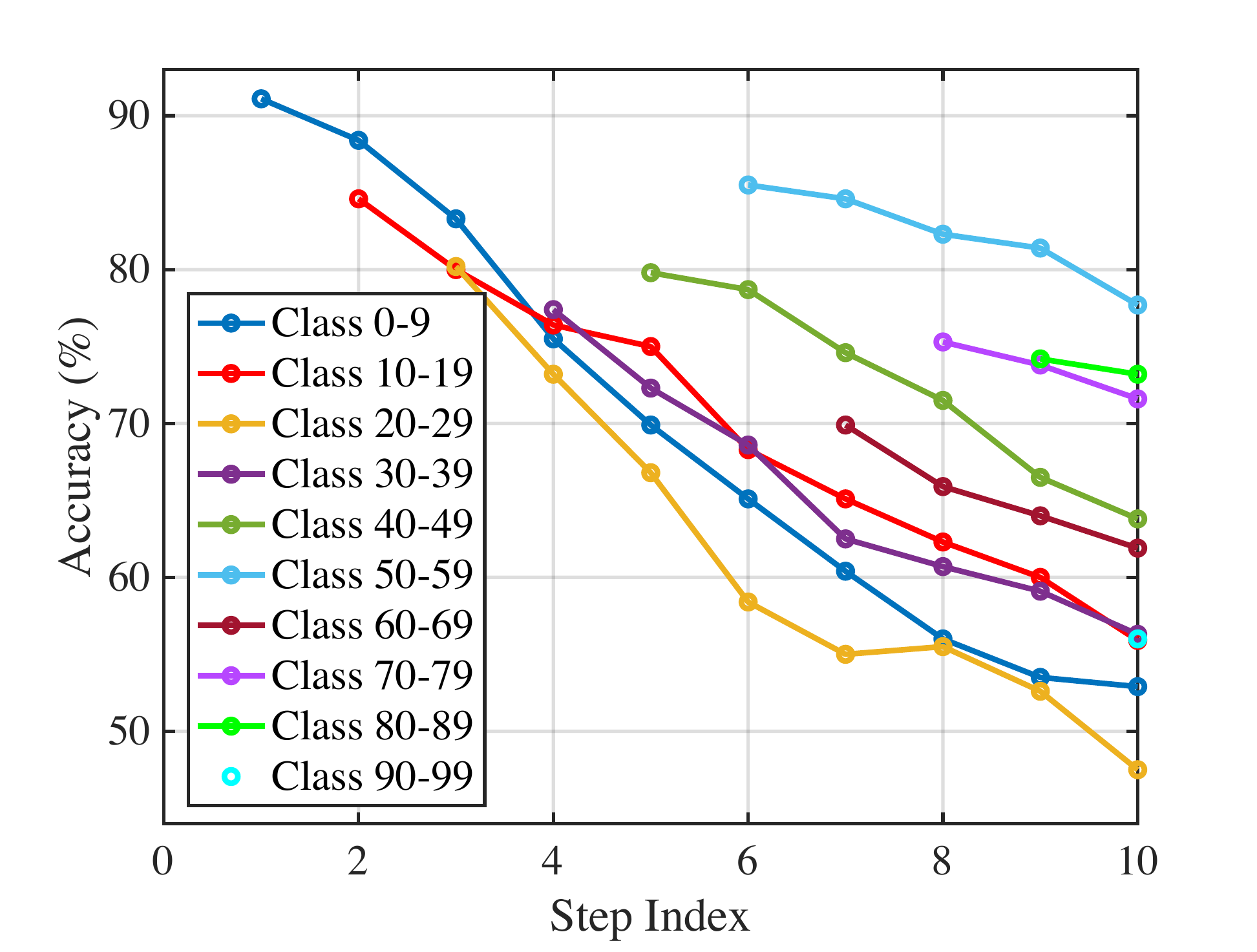}
\end{minipage}
\caption{Classification accuracy on CIFAR-100.}
\label{fig:cifar100}
\end{figure}

\paragraph{Accuracy vs. number of classes.} In Fig.~\ref{fig:cifar100}, we first show the comparison of incremental accuracies to some of the state-of-the-art methods including iCaRL~\cite{rebuffi2017icarl}, RPSNet~\cite{rajasegaran2019random}, BiC~\cite{hou2019learning} and WA~\cite{zhao2020maintaining}. One can observe that our H-LLL algorithm significantly outperforms the other methods and yields an average incremental accuracy~\cite{rebuffi2017icarl} of 73.8\%, while the second best approach (WA) only achieves 69.8\% accuracy.

\paragraph{Accuracy for different classes.} In order to gain deeper understanding of the H-LLL algorithm, we examine the accuracy of different class splits in each step. From Fig.~\ref{fig:cifar100}, we can see that the incremental accuracy for different class groups decreases in a slow and smooth way. This indicates that H-LLL is able to preserve knowledge of class concepts and effectively avoid catastrophic forgetting.

\section{Discussion}\label{section:discussion}

We study, theoretically and empirically, the efficiency of lifelong learning when tasks share a low-dimensional feature representation. We introduce a refinement algorithm and bound its representation and sample complexity, and prove a matching lower bound for the sample complexity (for any lifelong learning algorithm). 
Our results show that: (1) lifelong learning provably converges for nonlinear feature representations, (2) refinement has provable benefits, and (3) lifelong learning is an efficient approach to multi-class/multi-task learning. Our work also indicates that (a) refinement can be practical and can dynamically keep the dimension of the representation bounded and (b) remembering only a small subset of previous examples suffices for efficient lifelong learning.

These results raise further questions. In the general setting of nonlinear features, how can we guarantee that the refinement is efficient in terms of time complexity? Our experiments suggest that SGD does well in practice. One complication with nonlinear features is that even solving a single offline task is nontrivial and needs further assumptions. E.g., is there an efficient lifelong learning algorithm for two-layer ReLU networks under nice input distributions?

\paragraph{Acknowledgements.} This work was supported in part by NSF awards CCF-1909756, CCF-2007443 and CCF-2134105. Weiyang Liu is supported by a Cambridge-Tübingen Fellowship, an NVIDIA GPU grant, DeepMind and the Leverhulme Trust via CFI. We thank Le Song for helpful discussions.

\newpage
\bibliography{bibfile}

\newpage
\appendix
\section*{\LARGE Appendix}

\section{Formal Algorithm of LLL-RR}\label{section:appendix_algo_lllrr}
For completeness, we describe the formal algorithm of LLL-RR. Different from the basic LLL algorithm, we are memorizing a list $\tilde{\bm{w}}_1,\cdots,\tilde{\bm{w}}_{\hat{k}k_0}$ all along with the algorithm. Each time when we need to learn the new features $\tilde{\bm{w}}_{\hat{k}k_0+1},\cdots,\tilde{\bm{w}}_{\hat{k}k_0+k_0}$, we add them to the list, and feed the list to Algorithm \ref{algorithm:refinement} to get a new feature subspace. The formal algorithm is in Algorithm~\ref{algorithm:lll_rr}.

\begin{algorithm}[htbp]
    \caption{Lifelong Learning Algorithm with Representation Refinement (LLL-RR)}
    \label{algorithm:lll_rr}
    \KwIn{$d,m,k$, labeled examples of $m$ tasks
    , threshold parameters $\epsilon_{acc}, \epsilon$.}
    \begin{enumerate}
    \vspace{-2mm}
        \item Using data from the first task to learn a set of features $\tilde{\bm{W}}_1(\cdot)=(\tilde{\bm{w}}_1(\cdot),\cdots,\tilde{\bm{w}}_{k_0}(\cdot))^\top$ and a linear function $\tilde{\bm{c}}_1$ such that $\bm{x}\xrightarrow{}\sign(\tilde{\bm{c}}_1^\top
        \tilde{\bm{W}}_1(\bm{x}))$ has error smaller than $\epsilon_{acc}$. \\
        \tcc{Number of features $1\leq k_0\leq k$. For linear features, $k_0=1$.}
        \vspace{-2mm}
        \item[] Let $\tilde{k}=1$. 
        Set the feature subspace $\bm{V}_1=\tilde{\bm{W}}_1$, and the temporary features $\tilde{\bm{V}}_1=\tilde{\bm{W}}_1$. \\
        \vspace{-5mm}
        \item For the task $i=2,\cdots, m$
        \begin{itemize}
            \item Using the data from the $i$ task, attempt to learn the linear function $\tilde{\bm{c}_i}$ using the temporary features $\tilde{\bm{V}}_{i-1}$. 
            \item Check whether $\bm{x} \xrightarrow{}\sign( \tilde{\bm{c}}_i^\top\tilde{\bm{V}}_{i-1}(\bm{x}))$ has error less than $\epsilon$.
            \begin{enumerate}
                \item If yes, 
                set $\tilde{\bm{V}}_i=\tilde{\bm{V}}_{i-1}$.
                \tcp{Small error with current features.}
                \item Otherwise, learn a new set of features $\tilde{\bm{W}}_i(\cdot)$ and a linear function $\tilde{\bm{c}}_{i}$ such that the predictor $\bm{x}\xrightarrow{}\sign(\tilde{\bm{c}}_i^\top \tilde{\bm{W}}_i(\bm{x}))$ has error less than $\eps_{acc}$. 
                \\
                Update the feature subspace $\bm{V}_i=(\bm{V}_i; \tilde{\bm{W}}_i)$, and feed into Algorithm~\ref{algorithm:refinement}. It returns the refined subspace $\bm{V}'$. Set the temporary features $\tilde{\bm{V}}_i=\bm{V}'$.
                Let $\hat{k} = \hat{k}+1$.\\
            \end{enumerate}
        \end{itemize}
    \end{enumerate}
    \vspace{-7mm}
    \Return{$m$ predictors: $\bm{x} \xrightarrow{}\sign( \tilde{\bm{c}}_i^\top\tilde{\bm{V}}_{i}(\bm{x}))$
    , $1\leq i\leq m$.}
    \end{algorithm}

\section{Extensions to Task-Incremental Regression and Class-Incremental Learning}\label{section:class_increment}

\vspace{-2mm}
In our main text, we study the setting of solving $m$ tasks of binary classification incrementally. The classification error is defined as $err(\hat{l})= \mathbb{P}_{(\bm{x},y)\sim P}[\hat{l}(\bm{x}) \neq y]$. By Assumption \ref{equation:assumption}, we know the task error is small if and only if the parameters are close to each other. Now we would like to extend to task-incremental regression and class-incremental classification by connecting the parameter's $l_2$ distance to the error of the model.

\paragraph{Task-incremental regression.}
Consider the regression tasks shared with low-dimensional common features. For $i\in[m],y=\langle{\bm{c}_i^*}, \bm{\sigma}^*(\bm{x})\rangle+\epsilon_i$. The regression error is $err(\hat{l})=\mathbb{E}_{(\bm{x},y)\sim P}[\|\hat{l}(\bm{x})-y\|_2^2]$. We can further weaken our assumption to $c_1I \preccurlyeq \mathbb{E}[\bm{x}\bm{x}^\top]\preccurlyeq c_2I$ for $0 < c_1<c_2$. Then for any unit vector $\bm{u},\bm{v}$, we have
\[
\mathbb{E}_{\left(\bm{x},y\right)\sim P}\left[\left(\bm{u}^\top \bm{x}-\bm{v}^\top \bm{x}\right)^2\right]
=\mathbb{E}_{\left(\bm{x},y\right)\sim P}\left[(\bm{u}-\bm{v})^T\bm{x}\bm{x}^T\left(\bm{u}-\bm{v}\right)\right]
\]

So we get a lemma similar to Assumption~\ref{equation:assumption} that
$c_1\|\bm{u}-\bm{v}\|^2 \leq \mathbb{E}_{\bm{x}\in D} \|\bm{u}^\top \bm{x} - \bm{v}^\top \bm{x}\|^2 \leq c_2\|\bm{u}-\bm{v}\|^2$, which is sufficient for analysis.

\paragraph{Class-incremental classification.}
Let $X=\mathbb{R}^d$ be the input space and $Y = \{1,2,\cdots, m\}$ be the class labels. We assume that the labels can be recovered by passing the input through a linear/nonlinear layer and then taking the maximum of $m$ linear combinations. Formally, the label is given by
\[
\ell(\bm{x}) = \argmax_{i\in [m]} {\langle\bm{c}^*_i}, \bm{\sigma}^*(\bm{x})\rangle
\]
The classification error is $err(\hat{l}) = \mathbb{P}_{(\bm{x},y)\sim P}[\hat{l}(\bm{x}) \neq y]$. Noticing that, when we meet a new class, the classifier should determine whether the label belongs to the current class or not. In this sense, we regard the problem as a binary classification. To get the negative samples in the current class, we also need a small proportional of data from previous classes. Practically, we propose the heuristic lifelong learning (H-LLL) algorithm in Section~\ref{section:HLLL} to solve the class-incremental learning. Experiments in Section \ref{section:real_data_exp} complement our results.

\section{Another Approach -- Theoretical Guarantees for LLL}\label{section:appendix_theoretical_basic}
Here we give a simpler analysis for the basic LLL algorithm, along the lines of ~\cite{balcan2015efficient}. The result is weaker than Theorem~\ref{thm:vanilla_algo}, but we include the proof here for completeness, along with an extension to nonlinear features.

\begin{thm}[Basic LLL]\label{thm:lll_nonlinear_main}
Let $\gamma=c\epsilon$ and $\epsilon_{acc}$ s.t. $4k\frac{\epsilon_{acc}}{\gamma}+\gamma=c'\epsilon$ for sufficiently small constants $c,c'>0$. Assume that all targets share $k$ common features. Then, under Assumption \ref{equation:assumption}, and sequential presentation of the tasks in any order, the basic LLL algorithm will incrementally learn a representation of dimension $k$ for linear features and $k^2$ for nonlinear features with error at most $\eps$ on all tasks. The total number of samples used by the algorithm is $O(dk^2\log(k/\epsilon)/\epsilon^2+km\log(1/\epsilon)/\epsilon)=\tilde{O}(dk^2/\epsilon^2+mk/\epsilon)$ in the linear setting and a factor of $k$ higher in the nonlinear setting.
\end{thm}

Before we prove the theorem, we define the $\gamma$\textit{-separated} term. We use the definition $\gamma$\textit{-separated} from \cite{balcan2015efficient} that a subsequence of vectors $\bm{a}_{i_1},\bm{a}_{i_2},\cdots$ is $\gamma$-separated if for any $\bm{a}_{i_j}$, $\theta(\bm{a}_{i_j},\lspan(\bm{a}_{i_1},\cdots,\bm{a}_{i_{j-1}}))\geq \gamma$. Define the $\gamma$-effective dimension of $\bm{a}_1,\cdots,\bm{a}_m$ as the size of the largest $\gamma$-separated subsequence. Note that when $\gamma=0$, $\gamma$-effective dimension is exactly the dimension of the spanned subspace. We prove Theorem~\ref{thm:lll_nonlinear_main} by showing two facts: each target is far from the span of previous ones; we learn the new target accurately. (Lemma~\ref{lem:subspace}). We start with a helper lemma (Lemma~\ref{lem:theta_ineq}).

\begin{lemma}\label{lem:theta_ineq}
Let $\bm{w},\bm{v}$ be two unit vectors in $\R^d$ and $\bm{U}$ be a subspace. Then,
\[
\sin\theta(\lspan{(\bm{U},\bm{w})},\lspan{(\bm{U},\bm{v})})
\le \frac{\sin\theta(\bm{w},\bm{v})}{\max\{ \sin\theta(\bm{w},\bm{U}), \sin\theta(\bm{v},\bm{U})\}}
.
\]
\end{lemma}
\begin{proof}
If $\bm{w}\in \bm{U}$ or $\bm{v}\in \bm{U}$, $\theta(\lspan{(\bm{U},\bm{w})},\lspan{(\bm{U},\bm{v})})=0$. The inequality becomes trivial. Now we assume that $\bm{w},\bm{v}\notin \bm{U}$. From the symmetry of $\bm{w}$ and $\bm{v}$, we prove the following and then replacing $\bm{v}$ with $\bm{w}$ leads to the original inequality.
$$\sin\theta( \lspan{(\bm{U}, \bm{w})}, \lspan{(\bm{U},\bm{v})})\leq \frac{\sin\theta(\bm{w},\bm{v})}{\sin\theta(\bm{v},\bm{U})}
$$
By definition, $\exists \bm{x}\in\lspan{(\bm{U},\bm{w})}$ s.t. $\theta (\bm{x}, \lspan{(\bm{U},\bm{v})}) = \theta(\lspan{(\bm{U},\bm{w})}, \lspan{(\bm{U},\bm{v})})$. Here $\bm{x}$ is a combination of $\bm{u}_1$ and $\bm{w}$, $\bm{u}_1$ is some vector in $\bm{U}$. Using the fact that $\theta(\bm{x},\lspan{(\bm{U},\bm{v})}) \leq \theta (\bm{x}, \lspan{(\bm{u}_1,\bm{v})}) \leq \theta( \lspan{(\bm{u}_1, \bm{w})}, \lspan{(\bm{u}_1,\bm{v})})$ and $\theta(\bm{v},\bm{U}) \leq \theta(\bm{v},\bm{u}_1)$, it is sufficient to prove that 
$$\sin\theta( \lspan{(\bm{u}_1, \bm{w})}, \lspan{(\bm{u}_1,\bm{v})})\leq \frac{\sin\theta(\bm{w},\bm{v})}{\sin\theta(\bm{v},\bm{u}_1)}
$$
Denote $\alpha = \theta(\lspan{(\bm{u}_1, \bm{w})}, \lspan{(\bm{u}_1,\bm{v})})$, $\beta = \theta(\bm{v},\bm{u}_1)$. WLOG we assume $\bm{u}_1=(1,0,0)$ and $\lspan{(\bm{u}_1,\bm{w})}$ is the $x$-$y$ plane. Then we can write $\bm{v}=\cos(\beta)\bm{u}_1+\sin(\beta)\bm{v}_1$, where $\bm{v}_1=(0,\cos(\alpha),\sin(\alpha))$. Since $\sin\theta(\bm{w},\bm{v}) \geq d(\bm{v}, x\text{-}y \text{ plane}) = \sin(\alpha)\sin(\beta)$, we get the lemma proved.

\end{proof}

\begin{lemma}[Kernel Subspace]\label{lem:subspace}
Let $\bm{U}_k,\bm{V}_k$ be two subspaces of $\R^d$. Let $\bm{U}_k = \lspan\{\bm{y}_1^*,\cdots, \bm{y}_k^*\}$, $\bm{V}_k=\lspan\{\bm{y}_1,\ldots, \bm{y}_k \}$. Let $\epsilon, \gamma \ge 0$ and $\epsilon \le \gamma^2/(10k)$. 
Assume that 
\begin{enumerate}
\item $\sin\theta(\bm{y}_i, \lspan\{\bm{y}_1, \ldots, \bm{y}_{i-1}\}) \ge \gamma$, for $i=2,\cdots,k$.  
\item $\sin\theta(\bm{y}_i,\bm{y}^*_i) \le \eps$, for $i=1,\cdots,k$.
\end{enumerate}
Then we have $\sin\theta(\bm{U}_k,\bm{V}_k)\leq 2k\eps/\gamma$. In other words, for any point $\bm{y}^* \in \bm{U}_k$, there is a point $\bm{y}\in \bm{V}_k$ s.t.
\[
\sin\theta{(\bm{y}^*,\bm{y})} \le \frac{2\eps k}{\gamma}.
\]
\end{lemma}
\begin{proof}
Here we use the strong induction on a stronger version of the conclusion where $\bm{U}_k = \lspan\{\bm{W}, \bm{y}_1^*,\cdots, \bm{y}_k^*\}$, $\bm{V}_k=\lspan\{\bm{W}, \bm{y}_1,\ldots, \bm{y}_k \}$ for some fixed subspace $\bm{W}$. The base case is $k=1$. This follows directly from Lemma \ref{lem:theta_ineq} with $\bm{U}=\bm{W}, \bm{w} = \bm{y}, \bm{v} = \bm{y}^*$. Now we prove the induction step on $k$ with strong hypothesis. Let $\bm{U}_k' = \lspan(\bm{U}_{k-1}, \bm{y}_k)$. By Lemma~\ref{lem:theta_ineq} and induction hypothesis, we have
\[
\sin\theta(\bm{U}_k,\bm{V}_k)
\leq \sin \theta (\bm{U}_k,\bm{U}_k') + \sin\theta (\bm{U}_k',\bm{V}_k)
\leq \frac{\sin\theta(\bm{y}_k,\bm{y}_k^*)}{\sin\theta(\bm{y}_k,\bm{U}_{k-1})} + \frac{2(k-1)\epsilon}{\gamma}
\]
By triangle inequality and induction hypothesis, we further have
\[
\sin\theta(\bm{y}_k,\bm{U}_{k-1})
\geq \sin\theta(\bm{y}_k,\bm{V}_{k-1}) - \sin\theta (\bm{V}_{k-1}, \bm{U}_{k-1})
\geq \gamma - \frac{2\epsilon(k-1)}{\gamma}
\]
Combining the two inequalities, we get
\[
\sin\theta(\bm{U}_k,\bm{V}_k)
\leq \frac{\epsilon}{\gamma - \frac{2\epsilon(k-1)}{\gamma}} + \frac{2(k-1)\epsilon}{\gamma}
=\frac{\epsilon}{\gamma}\left(\frac{\gamma^2}{\gamma^2-2(k-1)\epsilon}+2(k-1)\right)
\leq  \frac{2k\epsilon}{\gamma}
\]

\end{proof}

Now we put them together to analyze Algorithm~\ref{algorithm:basic_lll}.

\begin{proof}
We consider the kernel of nonlinear features $\bm{\sigma}(\bm{x})$. These features live in a potentially infinite-dimensional space (or exponential in $d$ dimensional space if, e.g., the input is from the Boolean hypercube). 
Let $\bm{U}$ be the span of the nonlinear features (viewed as vectors) in the model used to label data.The $\gamma$-effective dimension of $\bm{U}$ is at most $k$. Let $y_i^*={\bm{c}_i^*}^\top \bm{\sigma}^*(\bm{x})=\bm{a}_i^*(\bm{x}),y_i=\bm{c}^\top_i\bm{\sigma}(\bm{x})=\bm{a}_i(\bm{x})$. WLOG let's assume $\bm{a}_i,\bm{a}_i^*$ be vectors of unit length. From the algorithm, if the current task $i$ has already achieved $\epsilon$ error by current features, it's done. Otherwise, we learn a new set of $k_0$ features and a linear combination whose error is at most $\epsilon_{acc}$. Denote the indices of tasks that we learn new features as $i_1,i_2,\cdots,i_{\tilde{k}}$. Encountering the task $i_{\hat{k}}$, denote $\bm{V}_{\hat{k}}=\lspan(\bm{a}_{i_1},\cdots,\bm{a}_{i_{\hat{k}}}), \bm{U}_{\hat{k}}=\lspan(\bm{a}^*_{i_1},\cdots,\bm{a}_{i_{\hat{k}}})$. We will prove by induction that for any $\hat{k}\in[\tilde{k}]$, (1) $\theta(\bm{a}_{i_{\hat{k}}},\bm{V}_{\hat{k}-1}) \geq \gamma$; (2) $\theta(\bm{a}^*_{i_{\hat{k}}},\bm{U}_{\hat{k}-1})\geq \gamma$.

The base case $\hat{k}=1$ holds immediately. For the inductive step $\hat{k}>1$, the task $i_{\hat{k}}$ cannot achieve $\epsilon$ error with the current features $\bm{V}_{\hat{k}-1}$. By Assumption~\ref{equation:assumption}, $\theta(\bm{a}^*_{i_{\hat{k}}}, \bm{V}_{\hat{k}-1})\geq \epsilon/c_2$. After learning a new set of features to ensure error less than $\epsilon_{acc}$, we know  there is a new linear combination $\bm{a}_{i_{\hat{k}}}$ such that $\theta(\bm{a}^*_{i_{\hat{k}}},\bm{a}_{i_{\hat{k}}}) \leq \epsilon_{acc}/c_1$. So by triangle inequality,
\[  
\theta\left(\bm{a}_{i_{\hat{k}}}, \bm{V}_{\hat{k}-1}\right) \geq \epsilon/c_2 - \epsilon_{acc}/c_1 \geq \gamma
\]
So we have shown that (1) holds for $i_{\hat{k}}$. 

To prove (2), we suppose for contradiction that $\theta(\bm{a}^*_{i_{\hat{k}}},\bm{U}_{\hat{k}-1})< \gamma$. From induction hypothesis, for any $j\in[\hat{k}-1]$, $\sin\theta(\bm{a}_{i_j}, \bm{V}_{j-1}) \geq \gamma/2$. By construction, we also have for any $\sin\theta(\bm{a}_{i_j},\bm{a}^*_{i_j}) \leq \epsilon_{acc}/c_1$. Apply Lemma~\ref{lem:subspace}, we have
$\theta(\bm{U}_{\hat{k}-1},\bm{V}_{\hat{k}-1})\leq 8\eps_{acc}k/\gamma$. By triangle inequality, we further have
\[
\theta\left(\bm{a}^*_{i_{\hat{k}}}, \bm{V}_{\hat{k}-1}\right) 
\leq \theta\left(\bm{a}^*_{i_{\hat{k}}}, \bm{U}_{\hat{k}-1}\right) + \theta\left(\bm{U}_{\hat{k}-1}, \bm{V}_{\hat{k}-1}\right)
\leq \gamma+4\eps_{acc}k/\left(c_1\gamma\right) \leq \eps/c_2
\]
By Assumption~\ref{equation:assumption}, there exists $\bm{b}_{i_{\hat{k}}}\in \bm{V}_{\hat{k}-1}$ with error less than $\epsilon$, and thus leads to contradiction. So (2) is also proved. Furthermore, since we have assume that the $\gamma$-effective dimension of the true targets is at most $k$, we have $\tilde{k}\leq k$. So the size of the internal representation is $k'=O(kk_0)$.

The sample complexity for learning one task in $d$-dimension up to error $\eps$ is $O(dk_0\log(1/\epsilon)/\eps)$. Here we learn $O(k)$ such tasks. All other tasks can be learned using the features of dimension $O(kk_0)$. Therefore the total sample complexity is $O(dkk_0/\eps_{acc}\log(1/\eps_{acc})+kk_0m\log(1/\eps)/\eps)=\tilde{O}(dk^2k_0/\epsilon^2+kk_0m/\eps)$.

\end{proof}

\end{document}